\documentclass[10pt,twocolumn,letterpaper]{article}

\usepackage[pagenumbers]{cvpr} 

\usepackage[accsupp]{axessibility}

\usepackage{microtype}
\usepackage{graphicx}
\usepackage{booktabs} 
\usepackage{multirow}
\usepackage{colortbl}
\usepackage{placeins}

\usepackage{bm}
\usepackage{nicefrac}
\usepackage{amsmath}
\usepackage{amssymb}
\usepackage{mathtools}
\usepackage{amsthm}
\usepackage{soul}
\usepackage{chngcntr}
\usepackage{algpseudocode}

\usepackage{booktabs}
\usepackage{multirow}
\usepackage[normalem]{ulem}
\useunder{\uline}{\ul}{}

\definecolor{mydarkorange}{HTML}{B86046}
\definecolor{codegreen}{rgb}{0,0.6,0}
\definecolor{codegray}{rgb}{0.5,0.5,0.5}
\definecolor{codepurple}{rgb}{0.58,0,0.82}
\usepackage[pagebackref,breaklinks,colorlinks]{hyperref}

\usepackage[utf8]{inputenc} 
\usepackage[T1]{fontenc}    
\usepackage{url}            
\usepackage{booktabs}       
\usepackage{amsfonts}       
\usepackage{nicefrac}       
\usepackage{microtype}      

\usepackage{pifont}

\usepackage[ruled,linesnumbered]{algorithm2e}
\SetKwComment{Comment}{/* }{ */}

\usepackage{enumerate}
\usepackage{siunitx}
\usepackage{enumitem}

\everymath{\displaystyle}
\newcommand{\zpw}{\mathcal{F}(S_{\text{PhaseWin}})}
\newcommand{\zopt}{\mathcal{F}(S_{\text{OPT}})}

\newcommand{\na}{\textemdash}
\definecolor{oursrow}{RGB}{235,235,252}
\newcommand{\best}[1]{\bfseries #1}
\newcommand{\mc}[1]{\multicolumn{1}{c}{#1}} 

\usepackage[many]{tcolorbox}
\tcbset{on line, boxsep=0.5pt, left=1.2pt,right=1.2pt, top=0.2pt,bottom=0.2pt,
        arc=1.2pt, boxrule=0.5pt, colframe=black, colback=oursrow}
\newtcbox{\rboxednum}{}

\theoremstyle{plain}
\newtheorem{theorem}{Theorem}
\newtheorem{proposition}{Proposition}

\theoremstyle{definition}
\newtheorem{definition}{Definition}

\theoremstyle{remark}
\newtheorem{remark}{Remark}

\counterwithin{theorem}{section}
\counterwithin{proposition}{section}
\counterwithin{lemma}{section}
\counterwithin{corollary}{section}
\counterwithin{definition}{section}
\counterwithin{assumption}{section}
\counterwithin{remark}{section}

\usepackage{alltt,capt-of}
\usepackage{tcolorbox}
\newenvironment{qbox}
{\begin{tcolorbox}[colback=white, width=\linewidth, center, left=2pt,right=2pt,top=4pt,bottom=4pt,boxrule=1.2pt,arc=4pt]}
{\end{tcolorbox}}

\definecolor{cvprblue}{rgb}{0.21,0.49,0.74}

\def\paperID{5847} 
\def\confName{CVPR}
\def\confYear{2026}

\title{PhaseWin Search Framework Enable Efficient Object-Level Interpretation}

\author{
Zihan Gu$^{1,2}$, Ruoyu Chen$^{1,2,\spadesuit}$, Junchi Zhang$^{3}$, Yue Hu$^{1,*}$, Hua Zhang$^{1,*}$,  Xiaochun Cao$^{4}$\\
\small$^{1}$Institute of Information Engineering, Chinese Academy of Sciences\\
\small$^{2}$School of Cyber Security, University of Chinese Academy of Sciences~~~~$^{3}$School of Mathematical Sciences, Fudan University\\
\small$^4$School of Cyber Science and Technology, Sun Yat-sen University\\
\small ~~~~~~~~~~~~~~~~~~~~~$\spadesuit$ Project Leader~~~~~~~~~~~~~~~~~~~~$*$ Corresponding Authors\\
\small\texttt{\{guzihan,chenruoyu,huyue,zhanghua\}@iie.ac.cn}~~~~~~~~\texttt{caoxiaochun@mail.sysu.edu.cn}
}

\begin{document}
\maketitle
\begin{abstract}
Attribution is essential for interpreting object-level foundation models. Recent methods based on submodular subset selection have achieved high faithfulness, but their efficiency limitations hinder practical deployment in real-world scenarios. To address this, we propose PhaseWin, a novel phase-window search algorithm that enables faithful region attribution with near-linear complexity. PhaseWin replaces traditional quadratic-cost greedy selection with a phased coarse-to-fine search, combining adaptive pruning, windowed fine-grained selection, and dynamic supervision mechanisms to closely approximate greedy behavior while dramatically reducing model evaluations. Theoretically, PhaseWin retains near-greedy approximation guarantees under mild monotone submodular assumptions. Empirically, PhaseWin achieves over 95\% of greedy attribution faithfulness using only 20\% of the computational budget, and consistently outperforms other attribution baselines across object detection and visual grounding tasks with Grounding DINO and Florence-2. PhaseWin establishes a new state of the art in scalable, high-faithfulness attribution for object-level multimodal models.
\end{abstract}
\section{Introduction}
\label{introduction}

\begin{figure}
    \centering
    \includegraphics[width=0.48\textwidth]{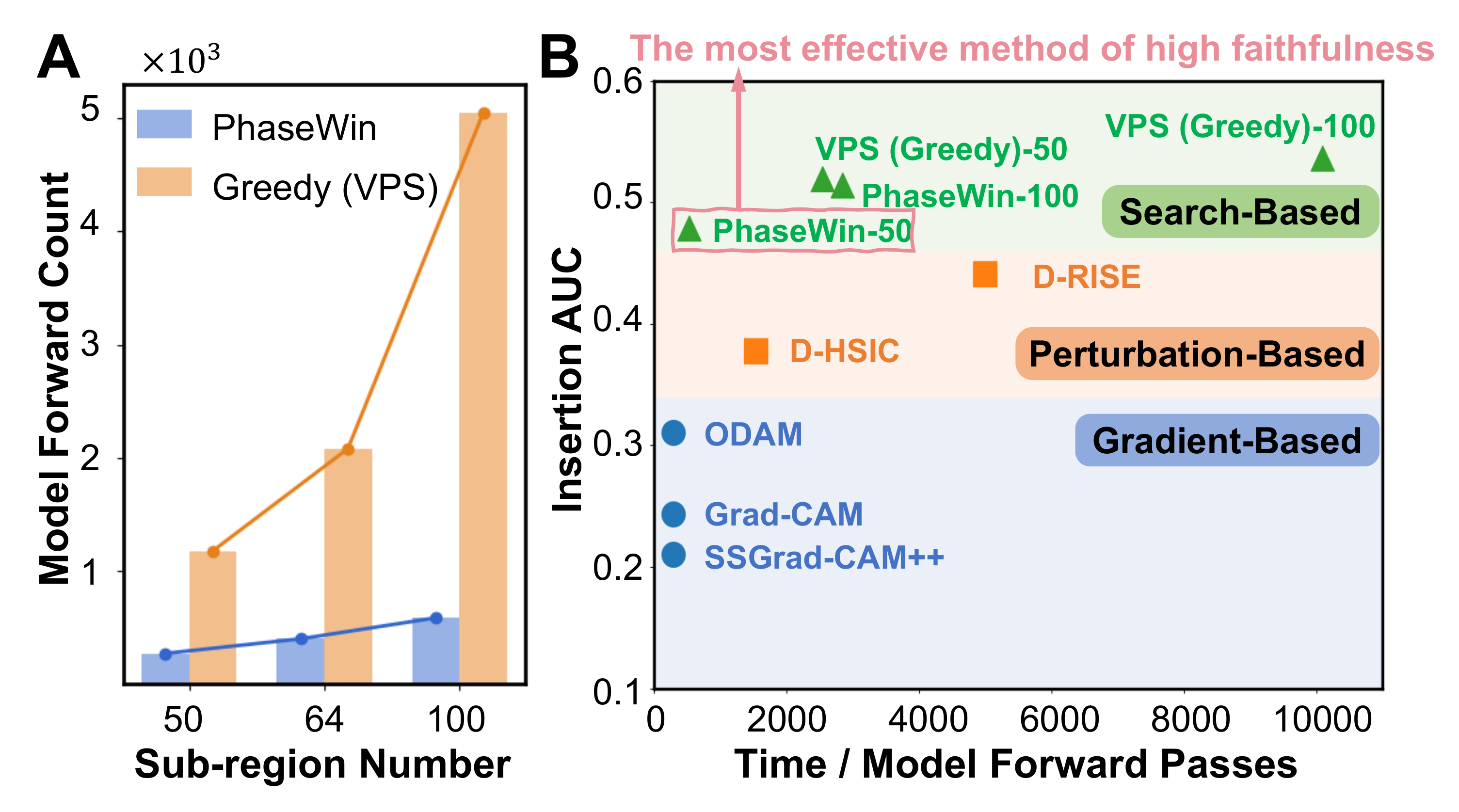}\vspace{-8pt}
    \caption{\textbf{A.} Comparison of model forward counts between VPS and PhaseWin (window Size fixed as $16$) across different subregion numbers. \textbf{B.} Comparison of Insertion AUC and computational cost among representative methods, where PhaseWin achieves near-VPS faithfulness with a fraction of the computational budget.} 
    \label{fig:intro}
\end{figure}

Understanding the behavior of large-scale foundation models~\citep{dwivedi2023explainable,gao2024going} is essential for reliable visual reasoning. Attribution methods~\citep{montavon2017explaining,yamauchi2024explaining} reveal which regions drive a model’s decision, supporting tasks such as debugging, failure diagnosis, bias detection, and safety auditing~\citep{miller2019evaluating,feng2021review,wilson2023safe,stocco2022thirdeye,shu2024informer,chen2025less}. In safety-critical applications, including autonomous driving and open-world perception~\citep{kuznietsov2024explainable,liu2023x}, faithful object-level attribution is key to building transparent and trustworthy AI systems.







Two families of attribution methods dominate current practice. Gradient-based approaches~\citep{zhao2024gradient_detector,yamauchi2024spatial} are efficient, but gradients in multimodal architectures often suffer from weak localization and cross-modal interference~\citep{selvaraju2020grad,jiang2024comparing}. Perturbation-based methods~\citep{petsiuk2018rise,petsiuk2021black} achieve stronger faithfulness by explicitly probing the model with masked inputs, yet their computational cost is prohibitively high~\citep{novello2022making,jiang2023diverse}. This makes them difficult to deploy in large-scale or real-time systems.

A recent breakthrough is the use of \emph{submodular modeling}. \textsc{LiMA}~\citep{chen2024less,chen2025less} in classification and VPS~\citep{chen2025vps} in object-level interpretation show that modeling attribution as maximizing a nearly submodular objective greatly improves faithfulness. VPS, in particular, delivers state-of-the-art attribution quality by applying a greedy region-selection process grounded in submodular optimization principles~\citep{edmonds1970submodular,fujishige2005submodular}. However, the greedy algorithm at its core evaluates every remaining region at each step. This results a \emph{quadratic} number of forward passes, which becomes the dominant bottleneck for scaling submodular-based attribution.

Figure~\ref{fig:intro}A illustrates this problem: the cost of greedy selection grows sharply with the number of subregions. Figure~\ref{fig:intro}B shows that although greedy-based attribution offers the highest faithfulness among representative paradigms, it remains substantially more expensive than gradient-based and mask-search-based methods. These trends raise a natural question:

\noindent\begin{qbox}
\begin{center}
\footnotesize\textbf{\textit{Can we retain the faithfulness of submodular greedy attribution while breaking its quadratic computational barrier?}}
\end{center}
\end{qbox}

To address this, we propose \textbf{Phase-Window (PhaseWin)} algorithm. PhaseWin rethinks greedy selection from a structural perspective. Instead of exhaustively scoring all candidates at every iteration, PhaseWin organizes the search into phases. Each phase begins by selecting an anchor region and using it to set adaptive thresholds, pruning the majority of low-potential candidates. A windowed fine-grained search then explores only a compact set of promising regions, guided by dynamic supervision that terminates unproductive phases and an annealed deferral mechanism that helps avoid poor local choices. This design preserves the behavior what makes greedy effective while reducing the number of true evaluations by an order of magnitude.

PhaseWin provides theoretical guarantees under standard monotone submodular conditions. More importantly, it delivers practical benefits for visual attribution. Across object detection and visual grounding tasks with Grounding DINO~\citep{liu2023grounding} and Florence-2~\citep{xiao2024florence}, and datasets including MS COCO, LVIS, and RefCOCO, PhaseWin achieves over 95\% of greedy attribution quality while using only about 20\% of its computational budget. As shown in Figure~\ref{fig:intro}, PhaseWin shifts the efficiency–faithfulness frontier, making high-fidelity submodular attribution computationally viable for real-world use.

Our contributions are summarized as follows:
\begin{itemize}[leftmargin=*, itemsep=0pt, topsep=0pt]
\item A structural improvement to greedy attribution. We show that the quadratic bottleneck of greedy region selection is not inherent. By reorganizing the scoring process into phased pruning and windowed local exploration, PhaseWin preserves greedy-like behavior while enabling near-linear scalability.
\item We introduce PhaseWin, a search procedure with adaptive thresholds, dynamic supervision, and annealed deferral. This algorithm provides theoretical near-greedy guarantees and significantly enhances the practicality of submodular modeling for object-level attribution.
\item Extensive empirical validation across multimodal foundation models demonstrates that PhaseWin consistently preserves over 95\% of greedy faithfulness while reducing computational cost to approximately 20\%. The method generalizes across diverse datasets (MS COCO, RefCOCO, LVIS) and architectures (Grounding DINO, Florence-2), establishing a new state of balance between efficiency and faithfulness.
\end{itemize}

\section{Related Work}
\textbf{Explaining Object Detectors.}
Explaining detector decisions is challenged by their intertwined localization and classification signals. Approaches range from adapting gradient-based attribution \citep{gudovskiy2018explain,selvaraju2020grad,zhao2024gradient_detector} and randomized perturbations \citep{petsiuk2018rise,petsiuk2021black} to refining Grad-CAM for spatial sensitivity \citep{yamauchi2022spatial,yamauchi2024spatial,chattopadhay2018grad}. While some methods explore diverse rationales at high computational cost \citep{jiang2023diverse}, recent state-of-the-art work uses causal search to generate high-fidelity explanations~\citep{chen2025vps}. Other studies compare architectures \citep{jiang2024comparing}, decompose representations \citep{gandelsman2024interpreting}, or analyze pixel collectives \citep{yamauchi2024explaining}, with broader XAI surveys providing context \citep{dwivedi2023explainable,gao2024going}.

\textbf{Submodular Function Maximization Algorithms.}
Submodular function optimization is an evolving theory widely applied to machine learning methods, such as network inference~\citep{rodriguez2012submodular}, object detection~\citep{song2014learning}, information collection~\citep{tschiatschek2014learning}, and diverse feature selection.~\citep{das2012selecting}. 
Our research also draws heavily on work that improves submodular function optimization or submodular functions, for which a variety of algorithms have been developed~\citep{edmonds1970submodular,horel2016maximization,hassani2017gradient,fujishige2005submodular,balkanski2019optimal}.
Therefore, the work on submodular function optimization often focuses on how to improve greedy algorithms, such as lazy greedy~\citep{minoux1978lazy}. This type of work often focuses on the search algorithm itself and does not fully utilize submodularity~\citep{leskovec2007lazy}.
Then, some methods of improving greedy algorithms by using submodularity appeared~\citep{jegelka2011fast,buchbinder2014submodular,wei2014fast,breuer2020fast,mirzasoleiman2015lazier}.
Since optimizing submodular functions doesn't necessarily mean optimizing AUC, this work can't be directly applied to attribution~\citep{jegelka2011fast,buchbinder2014submodular}. However, we considered how to exploit submodular properties~\citep{wei2014fast,breuer2020fast} and comprehensively designed our PhaseWin search algorithm, achieving a breakthrough in speed.
\section{Method}
\label{method}
\begin{figure*}[t!]
    \centering
    \includegraphics[width=\textwidth]{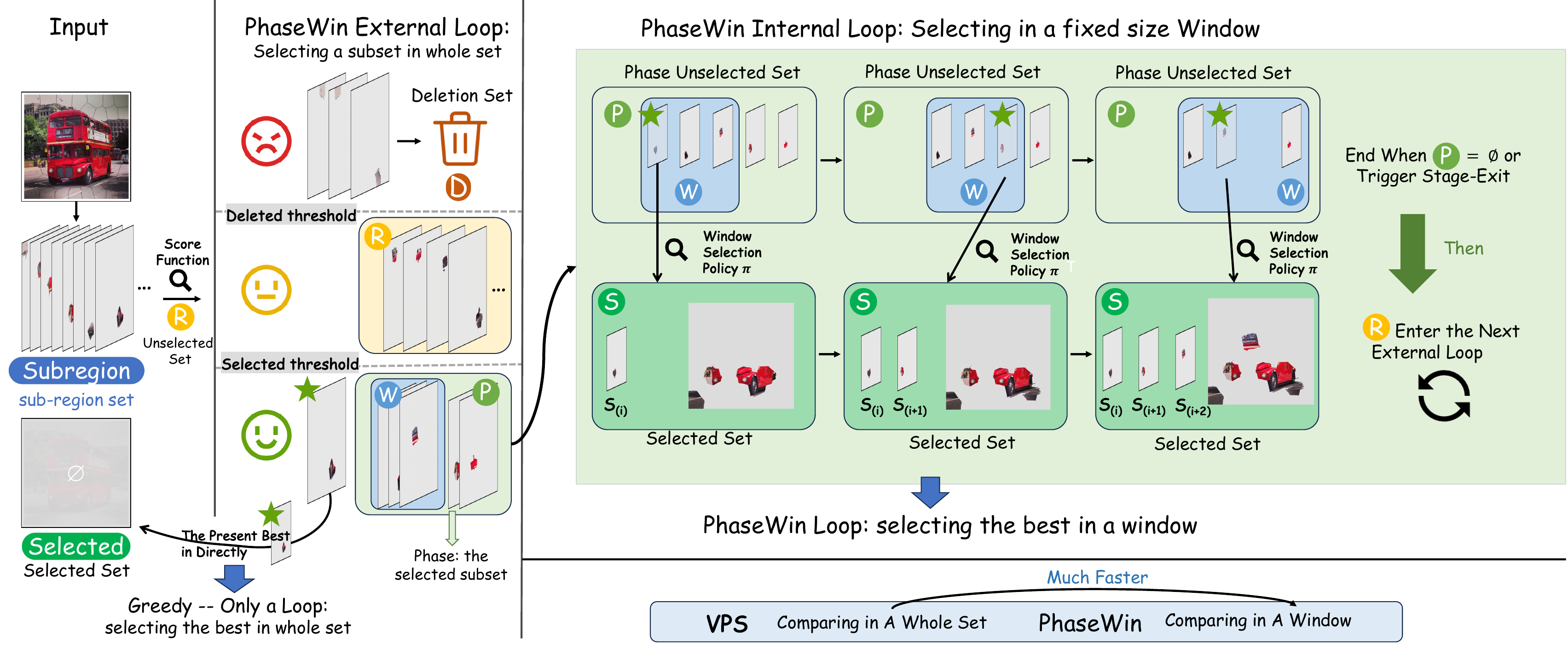} 
    \caption{\textbf{PhaseWin Workflow.} The algorithm alternates between (i) selecting an \textit{anchor} region, (ii) pruning uninformative regions using fixed-ratio thresholds, and (iii) applying a windowed fine-grained selection with dynamic supervision.}
    \label{fig:phasewin_workflow}
\end{figure*}

\subsection{Revisiting Attribution Target Problem}
\label{prob-fm}
The essence of the visual attribution task, as formalized in RISE~\citep{petsiuk2018rise}, is to progressively reveal regions of the input and measure the corresponding change in the model’s output. 
The objective is to maximize the area under the response curve of the model itself which quantifies how faithfully the model confidence rises as informative pixels are revealed. 
In object detection, this process extends naturally: given an input image $\mathbf{I}\in\mathbb{R}^{h\times w\times3}$ and a detection model $f(\cdot)$ producing bounding boxes, classes, and scores $f(\mathbf{I})=\{(b_i,c_i,s_i)\}_{i=1}^N$, the goal is to find the smallest ordered subset of image regions whose incremental insertion most efficiently recovers the model’s confidence on a target $(b_t,c_t,s_t)$.

If partitioning $\mathbf{I}$ into $m$ disjoint sub-regions
$
\mathcal{V} = \{\mathbf{I}^{s}_1, \dots, \mathbf{I}^{s}_m\},
$
and define an \emph{ordered} subset $\mathcal{S} = (s_1,\dots,s_k)$, where $s_i \in \mathcal{V}$.
The objective could be given as maximizing the area enclosed by $\mathcal{F}$ along the insertion trajectory of the ordered subset $\mathcal{S}$, such that the incremental inclusion of each element contributes to a steep and consistent rise in interpretability. The optimal ordered subset $\mathcal{S}^*$ as:
\begin{equation}\label{phasewin:object}
    \mathcal{S}^* =
    \arg\max_{\substack{\mathcal{S} = (s_1,\dots,s_k) \\ \mathcal{S} \subseteq \mathcal{V}}}
    \sum_{j=1}^{k} \frac{|s_j|}{A} \, f(\mathcal{S}_{:j}),
\end{equation}
where, $|s_j|$ denotes the pixel area of region $s_j$, $A$ is the total image area, and $\mathcal{S}_{:j}$ represents the first $j$ elements in the ordered subset $\mathcal{S}$. This formulation explicitly casts the task as an \emph{ordered subset optimization}.

To solve Eq.~\ref{phasewin:object}, prior work such as Visual Precision Search (VPS)~\citep{chen2025vps} 
formulates attribution as maximizing a near-submodular scoring function $\mathcal{F}$. 
This function integrates two complementary terms: 
a clue score $s_{\mathrm{clue}}$, which measures how well a region set supports the target detection, 
and a collaboration score $s_{\mathrm{colla}}$, which quantifies the synergistic importance of the selected regions 
by evaluating the degradation when they are removed. 
The combined objective is:
\[
\mathcal{F}(S,\mathbf{b}_{\mathrm{target}},c)
=
s_{\mathrm{clue}}(S,\mathbf{b}_{\mathrm{target}},c)
+
s_{\mathrm{colla.}}(S,\mathbf{b}_{\mathrm{target}},c),
\]
and is treated as approximately monotone submodular in practice. 
VPS adopts a standard greedy search to optimize this objective, 
yielding strong attribution faithfulness but at the cost of evaluating every remaining candidate at each step.

Although adjusting the curvature or regularization of $\mathcal{F}$ can further improve attribution faithfulness, the main limitation of this line of methods lies at the \emph{algorithmic} level: the greedy optimization procedure inherently incurs a quadratic number of model evaluations. 
For a fair comparison and to isolate the contribution of improving the \textit{search mechanism} itself, 
we retain the original VPS scoring function and focus on developing a more scalable greedy-style optimization algorithm---our proposed PhaseWin.

\subsection{Phase-Window Accelerated Search}
\label{sec:phasewin_alg}

For maximizing the ordered insertion-AUC objective, a naive greedy search that evaluates all remaining candidates at each step is theoretically optimal, but its $\mathcal{O}(m^2)$ scoring cost is prohibitive in practice. We propose the \textbf{Phase-Window (PhaseWin) Search}, an efficient greedy-style approximation that closely tracks greedy performance while reducing the number of expensive scoring function calls by an order of magnitude.

PhaseWin's acceleration stems from a phased, coarse-to-fine search strategy, illustrated in Figure~\ref{fig:phasewin_workflow}. The algorithm operates in phases. Each phase begins with a full evaluation to find a high-confidence \textit{anchor} region. The marginal gain of this anchor is then used as a reference to define two fixed-ratio thresholds: a selection threshold $\tau_{\mathrm{sel}} = \rho_{\mathrm{sel}} \cdot \Delta_{\mathrm{ref}}$ and a deletion threshold $\tau_{\mathrm{del}} = \rho_{\mathrm{del}} \cdot \Delta_{\mathrm{ref}}$, where $0 < \rho_{\mathrm{del}} < \rho_{\mathrm{sel}} < 1$ are constants and $\Delta_{\mathrm{ref}}$ is the anchor gain. Candidates with gains above $\tau_{\mathrm{sel}}$ enter a high-potential pool, while those below $\tau_{\mathrm{del}}$ are discarded. The remaining ambiguous candidates are deferred to future phases. This high-level process is detailed in Algorithm~\ref{alg:phasewin}.

\begin{algorithm}[h]
\caption{PhaseWin: Phase-Window Accelerated Search}
\label{alg:phasewin}

\renewcommand{\gets}{\leftarrow}

\KwIn{Candidate set $\mathcal{V}$, target size $k$, scoring function $\mathcal{F}(\cdot)$}
\KwOut{Ordered subset $S$}

$S \gets \emptyset$; \quad $\mathcal{R} \gets \mathcal{V}$; \quad $\Delta_{\mathrm{ref}} \gets \infty$\;

\While{$|S| < k$ \textbf{and} $\mathcal{R} \neq \emptyset$}{
    \tcp{Phase anchor selection}
    $g_r \gets \mathcal{F}(S \cup \{r\}) - \mathcal{F}(S)$ for all $r \in \mathcal{R}$\;
    $\alpha^\star \gets \arg\max_{r \in \mathcal{R}} g_r$\;
    $S \gets S \cup \{\alpha^\star\}$\;
    $\Delta_{\mathrm{ref}} \gets g_{\alpha^\star}$\;
    $\mathcal{R} \gets \mathcal{R} \setminus \{\alpha^\star\}$\;
    
    \tcp{Fixed-ratio thresholds}
    $\tau_{\mathrm{sel}} \gets \rho_{\mathrm{sel}} \cdot \Delta_{\mathrm{ref}}$\;
    $\tau_{\mathrm{del}} \gets \rho_{\mathrm{del}} \cdot \Delta_{\mathrm{ref}}$\;
    
    \tcp{Build high-potential pool}
    $\mathcal{P} \gets \emptyset$; \quad $\mathcal{R}_{\text{next}} \gets \emptyset$\;
    \For{$r \in \mathcal{R}$}{
        \lIf{$g_r \ge \tau_{\mathrm{sel}}$}{$\mathcal{P} \gets \mathcal{P} \cup \{r\}$}
        \lElseIf{$g_r \le \tau_{\mathrm{del}}$}{discard $r$}
        \lElse{$\mathcal{R}_{\text{next}} \gets \mathcal{R}_{\text{next}} \cup \{r\}$}
    }
    $\mathcal{R} \gets \mathcal{R}_{\text{next}}$\;
    
    \tcp{Windowed fine-grained search}
    $S_{\mathrm{phase}} \gets \texttt{WindowSelection}(\mathcal{P}, S, k, \mathcal{F}, \Delta_{\mathrm{ref}})$\;
    $S \gets S \cup S_{\mathrm{phase}}$\;
}

\KwRet{$S$}\;

\end{algorithm}

The core of our method lies in the \texttt{WindowSelection} subroutine, which performs a fine-grained search on the pruned high-potential pool $\mathcal{P}$. We first sort $\mathcal{P}$ by the cached gains $g_r$ and initialize a sliding window $W$ with the top-ranked candidates, while the remaining ones are stored in a queue $Q$. A window policy $\pi(\cdot)$ is then applied to select a subset $A \subseteq W$ for true evaluation. In practice, we consider two simple policies: (1) $\pi_{\mathrm{LG}}$, which picks only the top candidate in $W$, and (2) $\pi_{\mathrm{BA}}$, which selects all candidates whose cached gains are above a fixed-ratio cut-off based on the maximum gain in $W$.

For each candidate $\alpha \in A$, we compute its true marginal gain $\Delta_\alpha$ and evaluate it with two control mechanisms. First, a \textbf{stage-exit} rule compares $\Delta_\alpha$ to a reference $\Delta_{\mathrm{ref}}$ and terminates the phase early if $\Delta_\alpha < \theta \cdot \Delta_{\mathrm{ref}}$, avoiding unnecessary computation when returns become negligible. Otherwise, the candidate is further processed by an \textbf{annealing delay} mechanism, which decides whether to accept it immediately or defer its inclusion to encourage exploration. Accepted candidates are appended to the current solution $S$, and their gains update $\Delta_{\mathrm{ref}}$. After each update, the window $W$ is replenished from the queue $Q$ until either the target size $k$ is reached or no promising candidates remain.

\subsection{Theory Analysis}

Greedy search is both a curse and a constraint in the development of submodular function maximization algorithms:
it has long been proven to be the optimal and fastest method to achieve the best possible approximation under polynomial-time constraints.
We first restate the classic result as follows.

\begin{proposition}\label{pro}
    For maximizing a monotone submodular objective $\mathcal{F}:2^\mathcal{V}\to\mathbb{R}_+$ under a cardinality constraint $k$,
    let $S_{\mathrm{greedy}}$ denote the solution returned by the standard greedy algorithm
    and $S_{\mathrm{OPT}}$ denote the optimal subset of size $k$.
    Then the greedy algorithm achieves the optimal approximation ratio:
    \[
        \mathcal{F}(S_{\mathrm{greedy}})
        \;\geq\;
        \Bigl(1-\frac{1}{e}\Bigr)\,\mathcal{F}(S_{\mathrm{OPT}}),
    \]
    and no polynomial-time algorithm can surpass this bound unless $P=NP$~\citep{nemhauser1978analysis,fujishige2005submodular}.
\end{proposition}

Therefore, greedy selection serves as the \emph{de facto} gold standard,
and our analysis focuses on matching its empirical behavior while achieving substantial acceleration.
Our phase-window accelerated search (PhaseWin) is analogous to quicksort for sorting:
it is extremely fast in typical cases, yet it still offers explicit approximation guarantees with the phase-supervised early exit mechanism enabled.

\begin{theorem}[Approximation Guarantee]\label{thm}
    Let $S_{\mathrm{PhaseWin}}$ denote the solution returned by PhaseWin,
    and let $\theta\in[0,1)$ be an upper bound on the fraction of phases where early exits occur due to dynamic supervision.
    If the objective $\mathcal{F}$ is monotone submodular, then
    \[
        \mathcal{F}(S_{\mathrm{PhaseWin}})
        \;\geq\;
        \Bigl(1 - \frac{1}{e} - \mathbf{o}(1) \Bigr) \,\mathcal{F}(S_{\mathrm{OPT}}).
    \]
\end{theorem}

\begin{remark}
    This $\mathbf{o}(1)$ quantity is actually determined by $\tau_{sel},\tau_{del},k,\theta$. We have put the proof of this theorem in Appendix~\ref{appendix:proof}.
\end{remark}

\begin{table*}[t]
\centering
\small
\setlength{\tabcolsep}{5.5pt}
\renewcommand\arraystretch{1.15}
\caption{Comparison on three datasets for correctly detected or grounded samples using Grounding DINO.}
\label{tab:correct}
\resizebox{\textwidth}{!}{%
\begin{tabular}{
   l l
  !{\color{gray!40}\vrule}
  S[table-format=1.4] S[table-format=1.4] S[table-format=1.4] 
  S[table-format=1.4] S[table-format=1.4] S[table-format=1.4] 
  S[table-format=1.4]
  !{\color{gray!40}\vrule}
  S[table-format=1.4] S[table-format=1.4]
  !{\color{gray!40}\vrule}
  c c
}
\toprule
\multirow{3}{*}{\textbf{Datasets}} & \multirow{3}{*}{\textbf{Methods}} &
\multicolumn{7}{c!{\color{gray!40}\vrule}}{Faithfulness Metrics} &
\multicolumn{2}{c!{\color{gray!40}\vrule}}{Location Metrics} &
\multicolumn{2}{c}{\textbf{Efficiency Metrics}} \\
\cmidrule(lr){3-9}\cmidrule(lr){10-11}\cmidrule(lr){12-13}
& &
\mc{\shortstack{Ins.\\(↑)}} &
\mc{\shortstack{Del.\\(↓)}} &
\mc{\shortstack{Ins.~(class)\\(↑)}} &
\mc{\shortstack{Del.~(class)\\(↓)}} &
\mc{\shortstack{Ins.~(IoU)\\(↑)}} &
\mc{\shortstack{Del.~(IoU)\\(↓)}} &
\mc{\shortstack{Ave.~high.\\score~(↑)}} &
\mc{\shortstack{Point\\Game~(↑)}} &
\mc{\shortstack{Energy\\PG~(↑)}} &
\mc{\shortstack{$\text{MEC}_{\text{ave}}$\\~(↓)}} &
\mc{\shortstack{\textbf{A-C ratio}\\~(↑)}}
\\
\midrule
\multirow{9}{*}{\shortstack[l]{MS COCO~\citep{lin2014microsoft}\\(Detection task)}} & Grad-CAM~\citep{selvaraju2020grad}
& 0.2436 & 0.1526 & 0.3064 & 0.2006 & 0.6229 & 0.5324 & 0.5904 & 0.1746 & 0.1463 & \na & \na \\
& SSGrad-CAM++~\citep{yamauchi2022spatial} & 0.2107 & 0.1778 & 0.2639 & 0.2314 & 0.5981 & 0.5511 & 0.5886 & 0.1905 & 0.1293 & \na & \na \\
& D-RISE~\citep{petsiuk2018rise}        & 0.4412 & 0.0402 & 0.5081 & 0.0886 & 0.8396 & 0.3642 & 0.6215 & 0.9497 & 0.1850 & 5000 & 0.88 \\
& D-HSIC~\citep{novello2022making}        & 0.3776 & 0.0439 & 0.4382 & 0.0903 & 0.8301 & 0.3301 & 0.5862 & 0.7328 & 0.1861 & 1536 & 2.46 \\
& ODAM~\citep{zhao2024gradient}          & 0.3103 & 0.0519 & 0.3655 & 0.0894 & 0.7869 & 0.3984 & 0.5865 & 0.5431 & 0.2034 & \na & \na \\
\cmidrule(lr){2-13} 
& VPS(Greedy)-50~\cite{chen2025vps}          & 0.5195 & \best{0.0375} & 0.5941 & \best{0.0835} & 0.8480 & \best{0.3044} & 0.6591 & 0.9841 & \best{0.2046} & 2548.8 & 2.04 \\
\rowcolor{oursrow}
\cellcolor{white} & PhaseWin-50          & 0.4785 & 0.0424 & 0.5562 & 0.0898 & 0.8323 & 0.3116 & 0.6353 & \best{0.9894} & 0.1843 & \best{536.8} & \best{8.92} \\
\cmidrule(lr){2-13} 
& VPS(Greedy)-100~\citep{chen2025vps}          & \best{0.5459} & \best{0.0375} & \best{0.6204} & 0.0882 & \best{0.8581} & 0.3300 & \best{0.6873} & \best{0.9894} & \best{0.2046} & 10100& 0.54 \\
\rowcolor{oursrow}
\cellcolor{white} & PhaseWin-100         & 0.5141 & 0.0410 & 0.5890 & 0.0907 & 0.8505 & 0.3400 & 0.6644 & \best{0.9894} & 0.1628 & 2853.4 & 1.81 \\
\midrule
\multirow{9}{*}{\shortstack[l]{RefCOCO~\citep{kazemzadeh2014referitgame}\\(REC task)}} & Grad-CAM~\citep{selvaraju2020grad}
& 0.3749 & 0.4237 & 0.4658 & 0.5194 & 0.7516 & 0.7685 & 0.7481 & 0.2380 & 0.2171 & \na & \na \\
& SSGrad-CAM++~\citep{yamauchi2022spatial} & 0.4113 & 0.3925 & 0.5008 & 0.4851 & 0.7700 & 0.7588 & 0.7561 & 0.2820 & 0.2262 & \na & \na \\
& D-RISE~\citep{petsiuk2018rise}        & 0.6178 & 0.1605 & 0.7033 & 0.3396 & 0.8606 & 0.5164 & 0.8471 & 0.9400 & 0.2870 & 5000 & 1.24 \\
& D-HSIC~\citep{novello2022making}        & 0.5491 & 0.1846 & 0.6295 & 0.3509 & 0.8504 & 0.5120 & 0.7739 & 0.7900 & 0.3190 & 1536 & 3.57 \\
& ODAM~\citep{zhao2024gradient}          & 0.4778 & 0.2718 & 0.5620 & 0.3757 & 0.8217 & 0.6641 & 0.7425 & 0.6320 & 0.3529 & \na & \na \\
\cmidrule(lr){2-13} 
& VPS(Greedy)-50~\citep{chen2025vps}          & 0.7278 & \best{0.1240} & 0.7995 & 0.2473 & 0.8961 & \best{0.5053} & 0.8770 & \best{0.9580} & \best{0.3738} & 2290.6 & 3.18 \\
\rowcolor{oursrow}
\cellcolor{white} & PhaseWin-50          & 0.7013 & 0.1473 & 0.7794 & 0.2747 & 0.8862 & 0.5273 & 0.8654 & \best{0.9580} & 0.3530 & \best{630.1} & \best{11.13} \\
\cmidrule(lr){2-13} 
& VPS(Greedy)-100~\citep{chen2025vps}         & \best{0.7419} & 0.1250 & \best{0.8080} & \best{0.2457} & 0.9050 & 0.5103 & \best{0.8842} & 0.9460 & 0.3566 & 10100 & 0.73 \\
\rowcolor{oursrow}
\cellcolor{white} & PhaseWin-100          & 0.7377 & 0.1529 & 0.8046 & 0.2823 & \best{0.9054} & 0.5466 & 0.8813 & 0.9360 & 0.3076 & 3382.5 & 2.18 \\
\midrule
\multirow{9}{*}{\shortstack[l]{LVIS V1~\citep{gupta2019lvis} (rare)\\(Zero-shot det. task)}} & Grad-CAM~\citep{selvaraju2020grad}
& 0.1253 & 0.1294 & 0.1801 & 0.1814 & 0.5657 & 0.5910 & 0.3549 & 0.1151 & 0.0941 & \na & \na \\
& SSGrad-CAM++~\citep{yamauchi2022spatial} & 0.1253 & 0.1254 & 0.1765 & 0.1775 & 0.5800 & 0.5691 & 0.3504 & 0.1091 & 0.0931 & \na & \na \\
& D-RISE~\citep{petsiuk2018rise}        & 0.2808 & 0.0289 & 0.3348 & 0.0835 & 0.8303 & 0.3174 & 0.4289 & 0.9697 & 0.1462 & 5000 & 0.56 \\
& D-HSIC~\citep{novello2022making}        & 0.2417 & 0.0353 & 0.2912 & 0.0928 & 0.8187 & 0.3550 & 0.4044 & 0.8303 & 0.1730 & 1536 & 1.57 \\
& ODAM~\citep{zhao2024gradient}          & 0.2009 & 0.0410 & 0.2478 & 0.0844 & 0.7770 & 0.4082 & 0.3694 & 0.6061 & \best{0.2050} & \na & \na \\
\cmidrule(lr){2-13} 
& VPS(Greedy)-50~\citep{chen2025vps}          & 0.3411 & \best{0.0265} & 0.3995 & 0.0805 & 0.8372 & \best{0.2986} & 0.4654 & \best{0.9939} & 0.1439 & 2544.6 & 1.34 \\
\rowcolor{oursrow}
\cellcolor{white} & PhaseWin-50          & 0.3071 & 0.0303 & 0.3645 & 0.0893 & 0.8245 & 0.3097 & 0.4325 & \best{0.9939} & 0.1369 & \best{465.9} & \best{6.59} \\
\cmidrule(lr){2-13} 
& VPS(Greedy)-100~\citep{chen2025vps}          & \best{0.3695} & 0.0277 & \best{0.4275} & \best{0.0799} & \best{0.8479} & 0.3242 & \best{0.4969} & 0.9758 & 0.1785 & 10100 & 0.37 \\
\rowcolor{oursrow}
\cellcolor{white} & PhaseWin-100          & 0.3363 & 0.0309 & 0.3944 & 0.0839 & 0.8379 & 0.3374 & 0.4688 & 0.9697 & 0.1175 & 2726.8 & 1.23 \\
\bottomrule
\end{tabular}}
\end{table*}

\begin{table}[ht]
\centering
\caption{Comparison of approximation guarantee, complexity, and empirical acceleration.
$k$ denotes the subset size, $m$ the total candidate set size, and $\varepsilon$ the maximal early-exit ratio.}
\label{tab:approx_complexity}
\resizebox{\columnwidth}{!}{%
\begin{tabular}{lcccc}
\toprule
\textbf{Method} & \textbf{Approx. Guarantee} & \textbf{\# Marginal Evals} & \textbf{Complexity} & \textbf{Empirical Speedup} \\
\midrule
Greedy       & $(1 - 1/e)$ & $ \mathcal{O}(mk)$ & Quadratic & $1\times$ \\
Lazy Greedy  & $(1 - 1/e)$ & $\sim 0.7\,mk$ & Sub-quadratic & $\sim1.3\times$ \\
\textbf{PhaseWin} & $(1 - 1/e - \varepsilon)$ & $ \mathcal{O}(m)$ & Near-linear & $\mathbf{5\text{--}10\times}$ \\
\bottomrule
\end{tabular}}%
\end{table}

\paragraph{Time Complexity Analysis.}
Since the forward evaluation of the scoring function $\mathcal{F}(\cdot)$ dominates runtime,
we analyze complexity in terms of the number of calls to $\mathcal{F}$.

            With dynamic supervision, each phase aggressively prunes the candidate pool and
            probabilistically terminates when marginal gains diminish. Let $N_{\mathrm{exit}}$ be the expected number of early-exited phases.
            The expected number of calls is:
            \[
                \mathbb{E}[\#\text{calls}]
                \;=\;
                O\Bigl(
                    m \cdot f(\omega)
                    + m \cdot N_{\mathrm{exit}}
                \Bigr),
            \]
            where $w$ is the window size, and $f(\omega)$ is determined by the window policy $\pi$. For $\pi_{\mathrm{LG}}$ , $f(\omega)=\omega$, for $\pi_{\mathrm{BA}}$, $f(\omega)=\log(\omega)$, so the effective complexity approaches $\mathcal{O}(m)$ if $\omega << m$.

Thus, PhaseWin achieves \emph{greedy-level accuracy} while reducing the
number of expensive scoring calls by up to an order of magnitude in practice. The above theoretical analysis takes into account ideal situations and makes full use of the submodularity assumption. Our experiments confirm its high efficiency. The definitions of submodularity and supermodularity and their corresponding AUC curve properties are in Appendix~\ref{appendix:submodularity}.

\section{Experiments}
\label{exp}
\subsection{Experimental Setup}
We conduct a comprehensive evaluation of our method on object detection and referring expression comprehension (REC) tasks. The experiments are performed using two powerful object-level foundation models: Grounding DINO~\citep{liu2023grounding} and Florence-2~\citep{xiao2024florence}.

\textbf{Datasets and Baselines.}
We conduct experiments on three benchmarks. MS COCO 2017~\citep{lin2014microsoft} covers 80 object classes; we sample correctly detected, misclassified, and undetected instances for evaluation. LVIS V1~\citep{gupta2019lvis} spans 1,203 categories with 337 rare ones, where Grounding DINO~\citep{liu2023grounding} is used for zero-shot detection. RefCOCO~\citep{kazemzadeh2014referitgame} is adopted for the REC task, including both correct and incorrect grounding cases. We compare against representative attribution methods: gradient-based (Grad-CAM~\citep{selvaraju2020grad}, SSGrad-CAM++~\citep{yamauchi2022spatial}, ODAM~\citep{zhao2024gradient}), perturbation-based (D-RISE~\citep{petsiuk2018rise}, D-HSIC~\citep{novello2022making}), and the original greedy search(VPS~\citep{chen2025vps}), a greedy quadratic algorithm that serves as our acceleration target.

\textbf{Implementation Details.}
For PhaseWin, we adopt a default window size of 16 for 50 sub-regions and 32 for 100 sub-regions. 
Since the score function is not strictly monotonic submodular, we implement the stopping criterion using a numerically stable ratio-based formulation: 
$
\frac{S_{k-2}}{S_{k-1}} - \frac{S_{k-1}}{S_{k}} \leq \tau,
$ 
where we set $\tau=0.025$ for 50 sub-regions and $\tau=0.01$ for 100 sub-regions. 
Complete implementation details are provided in Appendix~\ref{app:impl}.

\subsection{Evaluation Metrics}
We evaluate the quality of attributions along three key axes: faithfulness, localization accuracy, and computational efficiency. This enables a holistic comparison of PhaseWin against all baselines.

\textbf{1. Faithfulness.}
We adopt standard insertion and deletion metrics to evaluate how well attribution maps reflect the model’s reasoning, applied to both classification confidence and IoU. We also report the highest box confidence ($\mathrm{IoU}>0.5$) and the Explaining Successful Rate (ESR) for failure cases.

\textbf{2. Localization Accuracy.}
We follow prior work and use the Point Game~\citep{zhang2018top} and Energy Point Game metrics~\citep{wang2020score} to quantify the alignment between attribution maps and ground-truth objects.

\textbf{3. Efficiency.}
We measure runtime efficiency using Model Evaluation Count (MEC), where one unit corresponds to a single forward pass. To combine accuracy and cost, we also report the Accuracy–Cost Ratio (AC-Ratio). These two indicators reflect the actual number of model calls required and the efficiency of the attribution algorithm, respectively, and are therefore not restricted by hardware. Details of the above metrics are provided in Appendix~\ref{app:metrics}.

\begin{table}[!t] 
\centering
\small
\setlength{\tabcolsep}{5.5pt}
\renewcommand\arraystretch{1.15}
\caption{Evaluation of faithfulness (Insertion/Deletion AUC) and efficiency metrics on MS COCO and RefCOCO validation sets (Florence-2).}
\label{tab:vps_t2_aligned_eff}
\resizebox{0.47\textwidth}{!}{%
\begin{tabular}{
   l l
  !{\color{gray!40}\vrule}
  S[table-format=1.4] S[table-format=1.4]
  !{\color{gray!40}\vrule}
  c c
}
\toprule
\multirow{2}{*}{\textbf{Datasets}} & \multirow{2}{*}{\textbf{Methods}} &
\multicolumn{2}{c!{\color{gray!40}\vrule}}{Faithfulness Metrics} &
\multicolumn{2}{c}{\textbf{Efficiency Metrics}} \\
\cmidrule(lr){3-4}\cmidrule(lr){5-6}
& &
\mc{\shortstack{Insertion (↑)}} &
\mc{\shortstack{Deletion (↓)}} &
\mc{\shortstack{$\text{MEC}_{\text{ave}}$ (↓)}} &
\mc{\shortstack{\textbf{A-C ratio} (↑)}} \\
\midrule
\multirow{4}{*}{\shortstack[l]{MS COCO\\(Detection task)}} 
& D-RISE & 0.7477 & 0.0972 & 5000 & 1.50 \\
& D-HSIC & 0.5345 & 0.2730 & \best{1536} & 3.48 \\
\cmidrule(lr){2-6} 
& VPS (Greedy)-50 & \best{0.7678} & 0.0550 & 2548.1 & 2.98 \\
& \cellcolor{oursrow} PhaseWin-50 
& \cellcolor{oursrow} 0.7615 
& \cellcolor{oursrow} \best{0.0474} 
& \cellcolor{oursrow} 2184.1 
& \cellcolor{oursrow} \best{3.49} \\
\midrule
\multirow{4}{*}{\shortstack[l]{RefCOCO\\(REC task)}} 
& D-RISE & 0.7922 & 0.3505 & 5000 & 1.24 \\
& D-HSIC & 0.7639 & 0.3560 & \best{1536} & \best{3.57} \\
\cmidrule(lr){2-6}
& VPS (Greedy)-50 & 0.8301 & \best{0.1159} & 2547.8 & 3.25 \\
& \cellcolor{oursrow} PhaseWin-50
& \cellcolor{oursrow} \best{0.8312} 
& \cellcolor{oursrow} 0.1205 
& \cellcolor{oursrow} 2349.1 
& \cellcolor{oursrow} 3.53 \\
\bottomrule
\end{tabular}} 
\end{table}

\subsection{Faithfulness Analysis on correct samples}
\subsubsection{Correct Interpretation on Grounding DINO}

We conduct faithfulness, locality, and efficiency tests on correct detection. Table~\ref{tab:correct} summarizes the results on correctly detected or grounded samples across three benchmarks. On the MS COCO detection task, PhaseWin substantially improves efficiency while maintaining a comparable level of faithfulness. Under the 50-region setting, it reduces the average model evaluations from 2548.8 to 536.8, a 4.7$\times$ reduction, with only a minor decrease in the Insertion score (0.5195 to 0.4785). This trade-off yields a marked improvement in the A-C ratio from 2.04 to 8.92. For the RefCOCO referring expression comprehension benchmark, PhaseWin achieves similar faithfulness to VPS (Greedy), with an Insertion score of 0.7013 versus 0.7278, while reducing model evaluations from 2290.6 to 630.1. This efficiency gain elevates the A-C ratio from 3.18 to 11.13, showing that PhaseWin produces high-quality attributions at a fraction of the cost. On the challenging LVIS v1 rare-class detection task, both VPS (Greedy) and PhaseWin show reduced overall faithfulness due to long-tail distributions. Nevertheless, PhaseWin lowers the computation demand from 2544.6 to 465.9 evaluations in the 50-region setting, improving the A-C ratio from 1.34 to 6.59. These results highlight that the efficiency benefits of PhaseWin become particularly valuable in computationally intensive scenarios, making attribution analysis more practical at scale.

\subsubsection{Correct Interpretation on Florence-2}

Table~\ref{tab:vps_t2_aligned_eff} reports results on MS COCO and RefCOCO when using Florence-2 as the backbone. Across both datasets, PhaseWin achieves faithfulness scores that are highly comparable to VPS (Greedy). On MS COCO, PhaseWin attains an Insertion score of 0.7615 versus 0.7678 from VPS (Greedy), with a slightly lower Deletion value (0.0474 vs.\ 0.0550). Similarly, on RefCOCO, PhaseWin produces an Insertion of 0.8312 against 0.8301 from VPS (Greedy), with a minor increase in Deletion. These results indicate that the acceleration strategy preserves the fidelity of VPS (Greedy) almost entirely.
When contrasted with perturbation-based baselines, PhaseWin consistently delivers higher faithfulness while requiring fewer model evaluations than D-RISE, and achieves efficiency comparable to D-HSIC but with stronger interpretability. The A-C ratio also reflects this balance: PhaseWin improves upon VPS (Greedy) (3.49 vs.\ 2.98 on COCO; 3.53 vs.\ 3.25 on RefCOCO), showing more favorable faithfulness-to-cost trade-offs.
It is worth noting that the acceleration gains are less pronounced compared to Grounding DINO. Florence-2 exhibits behavior that is nearly globally supermodular, while our acceleration relies on exploiting local submodularity. As discussed in Appendix~\ref{appendix:submodularity}, this structural property limits the extent of achievable speedup. Nevertheless, PhaseWin remains a strong alternative to VPS (Greedy), offering similar interpretability at reduced computational cost and outperforming other baselines across both benchmarks.

\subsection{Failures Interpretation}
\subsubsection{REC Failures Interpretation}

\begin{table}[!t]
\centering
\small
\setlength{\tabcolsep}{5.5pt}
\renewcommand\arraystretch{1.15}
\caption{RefCOCO (REC task): Faithfulness metrics and efficiency (Grounding DINO).}
\label{tab:refcoco_faith_eff}
\resizebox{0.47\textwidth}{!}{%
\begin{tabular}{
   l l
  !{\color{gray!40}\vrule}
  S[table-format=1.4] S[table-format=1.4] S[table-format=1.4]
  !{\color{gray!40}\vrule}
  c c
}
\toprule
\multirow{2}{*}{\textbf{Datasets}} & \multirow{2}{*}{\textbf{Methods}} &
\multicolumn{3}{c!{\color{gray!40}\vrule}}{Faithfulness Metrics} &
\multicolumn{2}{c}{\textbf{Efficiency Metrics}} \\
\cmidrule(lr){3-5}\cmidrule(lr){6-7}
& &
\mc{\shortstack{Ins. (↑)}} &
\mc{\shortstack{Ins.~(class) (↑)}} &
\mc{\shortstack{Ave.~high.score~ (↑)}} &
\mc{\shortstack{$\text{MEC}_{\text{ave}}$ (↓)}} &
\mc{\shortstack{\textbf{A-C ratio} (↑)}} \\
\midrule
\multirow{8}{*}{\shortstack[l]{RefCOCO\\(REC task)}}
& Grad-CAM        & 0.1536 & 0.2794 & 0.3295 & \na  & \na  \\
& SSGrad-CAM++    & 0.1590 & 0.2837 & 0.3266 & \na  & \na  \\
& D-RISE          & 0.3486 & 0.4787 & 0.6096 & 5000 & 1.21 \\
& D-HSIC          & 0.2274 & 0.3488 & 0.4495 & 1536 & 2.92 \\
& ODAM            & 0.1793 & 0.3001 & 0.3453 & \na  & \na  \\
\cmidrule(lr){2-7}
& VPS (Greedy)-100            & 0.4981 & 0.5990 & 0.7007 & 10100  & 0.69  \\
& \cellcolor{oursrow} PhaseWin-50
  & \cellcolor{oursrow} 0.4455
  & \cellcolor{oursrow} 0.5537
  & \cellcolor{oursrow} 0.6437
  & \cellcolor{oursrow} \best{614.4}
  & \cellcolor{oursrow} \best{10.48} \\
& \cellcolor{oursrow} PhaseWin-100
  & \cellcolor{oursrow} \best{0.5047}
  & \cellcolor{oursrow} \best{0.6023}
  & \cellcolor{oursrow} \best{0.7116}
  & \cellcolor{oursrow} 3164.4
  & \cellcolor{oursrow} 2.25 \\
\bottomrule
\end{tabular}}
\end{table}

Table~\ref{tab:refcoco_faith_eff} presents results on RefCOCO samples where Grounding DINO produces incorrect grounding. Compared with gradient-based baselines such as Grad-CAM and ODAM, both VPS (Greedy) and PhaseWin yield substantially higher insertion scores and average confidence, indicating that search-based attribution is better suited for recovering meaningful evidence under failure cases. Perturbation-based approaches like D-RISE and D-HSIC achieve moderate improvements, but remain less faithful overall.  
Between the two search variants, PhaseWin attains an attribution quality that is highly comparable to VPS (Greedy). Under the 100-region setting, PhaseWin slightly surpasses VPS (Greedy) in insertion and classification-based scores (0.5047 vs.\ 0.4981 and 0.6023 vs.\ 0.5990), while using fewer model evaluations (3164.4 vs.\ 10100). In the 50-region setting, PhaseWin achieves somewhat lower insertion metrics than VPS (Greedy) but with drastically reduced computational demand (614.4 vs.\ 10100 evaluations). This efficiency translates into a much higher A-C ratio, rising from 0.69 with VPS (Greedy) to 10.48 with PhaseWin.  
These results suggest that PhaseWin can provide nearly the same level of interpretability as the greedy search, while significantly reducing the computational cost. This advantage is especially valuable when analyzing mis-grounded instances, where large-scale evaluation would otherwise be prohibitive.

\subsubsection{Misclassified Detection Failures Interpretation}

\begin{table}[!t]
\centering
\small
\setlength{\tabcolsep}{5.5pt}
\renewcommand\arraystretch{1.15}
\caption{MS COCO and LVIS (misclassified samples): Faithfulness metrics and efficiency (Grounding DINO)}
\label{tab:miscls_faith_eff}
\vspace{-6pt}
\resizebox{0.47\textwidth}{!}{%
\begin{tabular}{
   l l
  !{\color{gray!40}\vrule}
  S[table-format=1.4] S[table-format=1.4] S[table-format=1.4] c
  !{\color{gray!40}\vrule}
  c c
}
\toprule
\multirow{2}{*}{\textbf{Datasets}} & \multirow{2}{*}{\textbf{Methods}} &
\multicolumn{4}{c!{\color{gray!40}\vrule}}{Faithfulness Metrics} &
\multicolumn{2}{c}{\textbf{Efficiency Metrics}} \\
\cmidrule(lr){3-6}\cmidrule(lr){7-8}
& &
\mc{\shortstack{Ins.~(↑)}} &
\mc{\shortstack{Ins.~(class)~(↑)}} &
\mc{\shortstack{Ave.~high.~score~(↑)}} &
\mc{\shortstack{ESR~(↑)}} &
\mc{\shortstack{$\text{MEC}_{\text{ave}}$~(↓)}} &
\mc{\shortstack{\textbf{A-C ratio}~(↑)}} \\
\midrule
\multirow{8}{*}{\shortstack[l]{MS COCO\\(Detection task)}}
& Grad-CAM        & 0.1091 & 0.1478 & 0.3102 & 38.38\% & \na  & \na  \\
& SSGrad-CAM++    & 0.0960 & 0.1336 & 0.2952 & 33.51\% & \na  & \na  \\
& D-RISE          & 0.2170 & 0.2661 & 0.3603 & 50.26\% & 5000 & 0.72 \\
& D-HSIC          & 0.1771 & 0.2161 & 0.3143 & 34.59\% & 1536 & 2.04 \\
& ODAM            & 0.1129 & 0.1486 & 0.2869 & 32.97\% & \na  & \na  \\
\cmidrule(lr){2-8}
& VPS (Greedy)-100            & \best{0.3357} & \best{0.3967} & \best{0.4591} & \best{69.73\%} & 10100  & 0.45  \\
& \cellcolor{oursrow} PhaseWin-50
  & \cellcolor{oursrow} 0.2614
  & \cellcolor{oursrow} 0.3198
  & \cellcolor{oursrow} 0.3770
  & \cellcolor{oursrow} 51.35\%
  & \cellcolor{oursrow} \best{477.3}
  & \cellcolor{oursrow} \best{7.90} \\
& \cellcolor{oursrow} PhaseWin-100
  & \cellcolor{oursrow} 0.3018
  & \cellcolor{oursrow} 0.3583
  & \cellcolor{oursrow} 0.4289
  & \cellcolor{oursrow} 63.78\%
  & \cellcolor{oursrow} 2595.0
  & \cellcolor{oursrow} 1.65 \\
\midrule
\multirow{8}{*}{\shortstack[l]{LVIS V1 (rare)\\(Zero-shot det. task)}}
& Grad-CAM        & 0.0503 & 0.0891 & 0.1564 & 12.50\% & \na  & \na  \\
& SSGrad-CAM++    & 0.0574 & 0.0946 & 0.1580 & 11.84\% & \na  & \na  \\
& D-RISE          & 0.1245 & 0.1647 & 0.2088 & 28.95\% & 5000 & 0.41 \\
& D-HSIC          & 0.0963 & 0.1247 & 0.1748 & 16.45\% & 1536 & 1.14 \\
& ODAM            & 0.0575 & 0.0954 & 0.1520 & 9.21\%  & \na  & \na  \\
\cmidrule(lr){2-8}
& VPS (Greedy)-100            & \best{0.1776} & \best{0.2190} & \best{0.2606} & \best{43.29\%} & 10100  & 0.26  \\
& \cellcolor{oursrow} PhaseWin-50
  & \cellcolor{oursrow} 0.1394
  & \cellcolor{oursrow} 0.1817
  & \cellcolor{oursrow} 0.2119
  & \cellcolor{oursrow} 36.63\%
  & \cellcolor{oursrow} \best{426.5}
  & \cellcolor{oursrow} \best{5.20} \\
& \cellcolor{oursrow} PhaseWin-100
  & \cellcolor{oursrow} 0.1475
  & \cellcolor{oursrow} 0.1845
  & \cellcolor{oursrow} 0.2296
  & \cellcolor{oursrow} 39.47\%
  & \cellcolor{oursrow} 2204.8
  & \cellcolor{oursrow} 1.04 \\
\bottomrule
\end{tabular}}
\end{table}

Table~\ref{tab:miscls_faith_eff} reports results on misclassified samples from MS COCO and LVIS. Gradient-based methods such as Grad-CAM and ODAM show limited utility in this setting, while perturbation-based baselines (D-RISE and D-HSIC) provide moderate improvements in insertion and class-specific scores. Both VPS (Greedy) and PhaseWin yield higher overall faithfulness, indicating that search-based approaches are better suited to reveal discriminative regions responsible for misclassification. 
On MS COCO, VPS (Greedy) achieves the strongest raw faithfulness metrics, with an Insertion score of 0.3357 and an ESR of 69.73\%. PhaseWin produces slightly lower attribution quality under both 50- and 100-region settings, but substantially reduces computational requirements. In particular, PhaseWin-50 lowers the average model evaluations from 10100 to only 477.3, raising the A-C ratio from 0.45 to 7.90. This demonstrates that PhaseWin can provide competitive interpretability while making failure case analysis far more efficient.
On LVIS rare-class detection, all methods perform worse due to the long-tail distribution, but the same trend holds. VPS (Greedy) delivers the highest insertion and ESR values, while PhaseWin achieves comparable results at a fraction of the computational cost. PhaseWin-50 reduces the model evaluations by over 20$\times$ compared to VPS (Greedy), yielding an A-C ratio of 5.20 versus 0.26. These results show that PhaseWin remains practical for large-scale misclassification analysis, where the quadratic cost of full VPS (Greedy) would be prohibitive.

\begin{figure*}[t!]
    \centering
    \includegraphics[width=\textwidth]{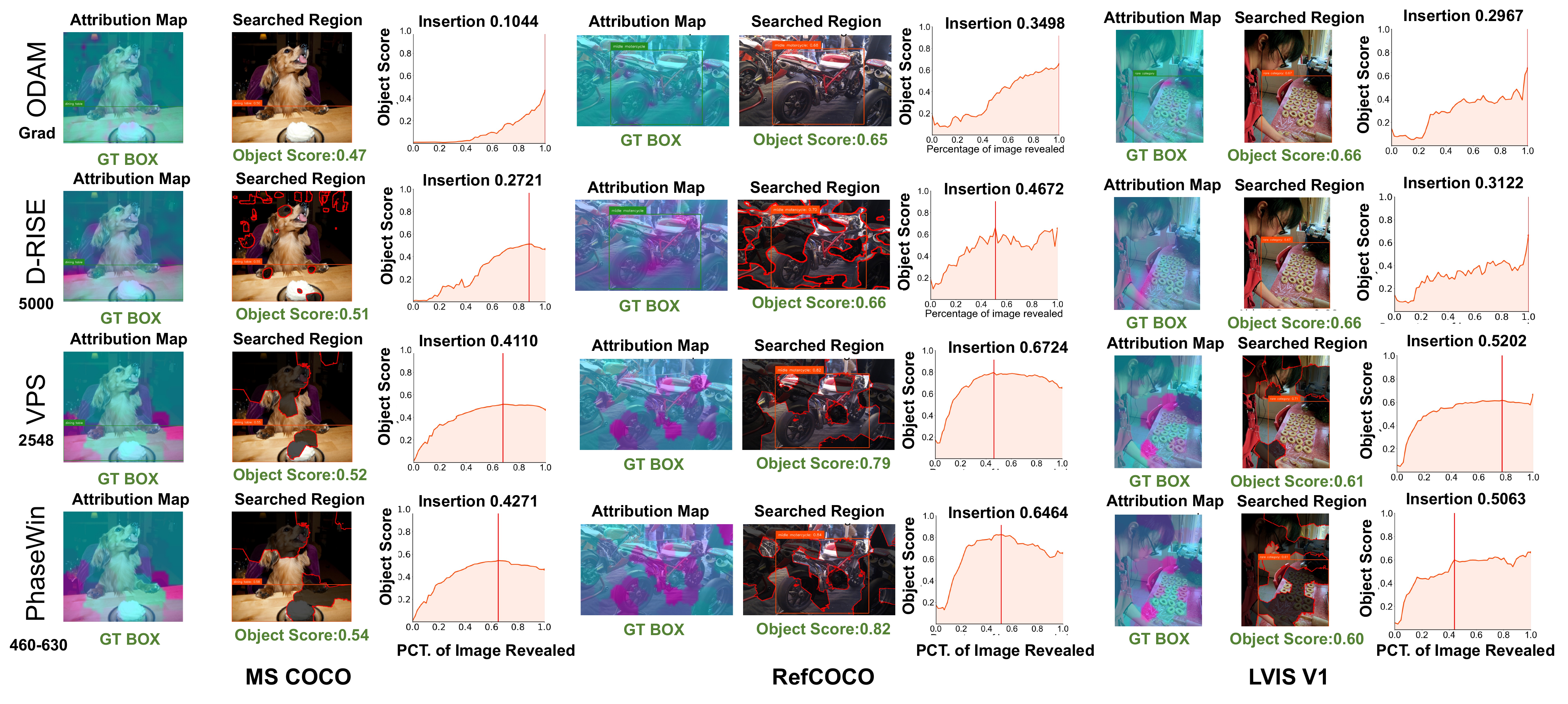}
    \caption{Visualization of correct attribution cases on MS COCO, RefCOCO, and LVIS V1. 
    Compared with ODAM and D-RISE, PhaseWin produces sharper and more faithful attributions. 
    It matches or even exceeds VPS (Greedy) in insertion-AUC while requiring only $\sim$20\% of its computational cost.}
    \label{fig:vis}
\end{figure*}

\subsubsection{Undetected Detection Failures Interpretation}

\begin{table}[!t]
\centering
\small
\setlength{\tabcolsep}{5.5pt}
\renewcommand\arraystretch{1.15}
\caption{MS COCO and LVIS (undetected failure samples): Faithfulness metrics and efficiency (Grounding DINO).}
\label{tab:undetected_faith_eff}
\vspace{-6pt}
\resizebox{0.47\textwidth}{!}{%
\begin{tabular}{
   l l
  !{\color{gray!40}\vrule}
  S[table-format=1.4] S[table-format=1.4] S[table-format=1.4] c
  !{\color{gray!40}\vrule}
  c c
}
\toprule
\multirow{2}{*}{\textbf{Datasets}} & \multirow{2}{*}{\textbf{Methods}} &
\multicolumn{4}{c!{\color{gray!40}\vrule}}{Faithfulness Metrics} &
\multicolumn{2}{c}{\textbf{Efficiency Metrics}} \\
\cmidrule(lr){3-6}\cmidrule(lr){7-8}
& &
\mc{\shortstack{Ins.~(↑)}} &
\mc{\shortstack{Ins.~(class)~(↑)}} &
\mc{\shortstack{Ave.~high.~score~(↑)}} &
\mc{\shortstack{ESR~(↑)}} &
\mc{\shortstack{$\text{MEC}_{\text{ave}}$~(↓)}} &
\mc{\shortstack{\textbf{A-C ratio}~(↑)}} \\
\midrule
\multirow{8}{*}{\shortstack[l]{MS COCO\\(Detection task)}}
& Grad-CAM        & 0.0760 & 0.1321 & 0.2153 & 16.44\% & \na  & \na  \\
& SSGrad-CAM++    & 0.0671 & 0.1151 & 0.2124 & 16.44\% & \na  & \na  \\
& D-RISE          & 0.1538 & 0.2260 & 0.2564 & 26.94\% & 5000 & 0.31  \\
& D-HSIC          & 0.1101 & 0.1716 & 0.1945 & 13.56\% & 1536 & 1.43  \\
& ODAM            & 0.0745 & 0.1350 & 0.2037 & 13.78\% & \na  & \na  \\
\cmidrule(lr){2-8}
& VPS (Greedy)-100            & 0.2102 & 0.3011 & 0.3014 & 41.33\% & 10100  & 0.21  \\
& \cellcolor{oursrow} PhaseWin-50
  & \cellcolor{oursrow} 0.1801
  & \cellcolor{oursrow} 0.2641
  & \cellcolor{oursrow} 0.2726
  & \cellcolor{oursrow} 33.78\%
  & \cellcolor{oursrow} \best{427.8}
  & \cellcolor{oursrow} \best{6.37} \\
& \cellcolor{oursrow} PhaseWin-100
  & \cellcolor{oursrow} \best{0.2156}
  & \cellcolor{oursrow} \best{0.3045}
  & \cellcolor{oursrow} \best{0.3289}
  & \cellcolor{oursrow} \best{44.44\%}
  & \cellcolor{oursrow} 2160.2
  & \cellcolor{oursrow} 1.52 \\
\midrule
\multirow{8}{*}{\shortstack[l]{LVIS V1 (rare)\\(Zero-shot det. task)}}
& Grad-CAM        & 0.0291 & 0.0689 & 0.0901 & 5.43\% & \na  & \na  \\
& SSGrad-CAM++    & 0.0292 & 0.0680 & 0.0897 & 5.24\% & \na  & \na  \\
& D-RISE          & 0.0703 & 0.1184 & 0.1312 & 18.73\% & 5000 & 0.26  \\
& D-HSIC          & 0.0516 & 0.0920 & 0.1168 & 13.48\% & 1536 & 0.76  \\
& ODAM            & 0.0283 & 0.0716 & 0.0851 & 4.68\% & \na  & \na  \\
\cmidrule(lr){2-8}
& VPS (Greedy)-100            & \best{0.1155} & \best{0.1886} & \best{0.1784} & \best{30.15\%} & 10100  & 0.18  \\
& \cellcolor{oursrow} PhaseWin-50
  & \cellcolor{oursrow} 0.0787
  & \cellcolor{oursrow} 0.1286
  & \cellcolor{oursrow} 0.1309
  & \cellcolor{oursrow} 17.04\%
  & \cellcolor{oursrow} \best{348.4}
  & \cellcolor{oursrow} \best{3.76} \\
& \cellcolor{oursrow} PhaseWin-100
  & \cellcolor{oursrow} 0.0942
  & \cellcolor{oursrow} 0.0069
  & \cellcolor{oursrow} 0.1552
  & \cellcolor{oursrow} 24.72\%
  & \cellcolor{oursrow} 1509.1
  & \cellcolor{oursrow} 1.03 \\
\bottomrule
\end{tabular}}
\vspace{-8pt}
\end{table}

Table~\ref{tab:undetected_faith_eff} shows results on MS COCO and LVIS samples where the objects are not detected. In this challenging setting, gradient-based baselines such as Grad-CAM and ODAM yield low insertion and class-specific scores, reflecting limited explanatory power. Perturbation-based methods (D-RISE and D-HSIC) offer some improvements, but remain costly or less stable. In contrast, both VPS (Greedy) and PhaseWin deliver more reliable attribution maps, better capturing the evidence that is missing in undetected cases.  
On MS COCO, PhaseWin achieves faithfulness that is close to VPS (Greedy), with Insertion and ESR values of 0.2156 and 44.44\% under the 100-region setting, slightly exceeding VPS (Greedy). More importantly, it reduces the average model evaluations from 10100 to 2160.2, raising the A-C ratio from 0.21 to 1.52. Under the 50-region setting, the efficiency advantage is even more pronounced, with only 427.8 evaluations required and an A-C ratio of 6.37.  
On LVIS rare-class detection, overall faithfulness scores are lower due to the long-tail distribution, but the same pattern holds. VPS (Greedy) achieves the highest raw insertion metrics, while PhaseWin provides comparable results with far fewer model evaluations. For example, PhaseWin-50 requires just 348.4 evaluations compared to 10100 for VPS (Greedy), lifting the A-C ratio from 0.18 to 3.76.  
These results show that PhaseWin offers an effective trade-off in undetected failure analysis: it maintains interpretability comparable to VPS while drastically reducing computational cost, enabling practical attribution studies even on large-scale failure cases.

\subsection{Speed and Precision Ablation}

\begin{figure}[!t]
    \centering
    \includegraphics[width=0.9\columnwidth]{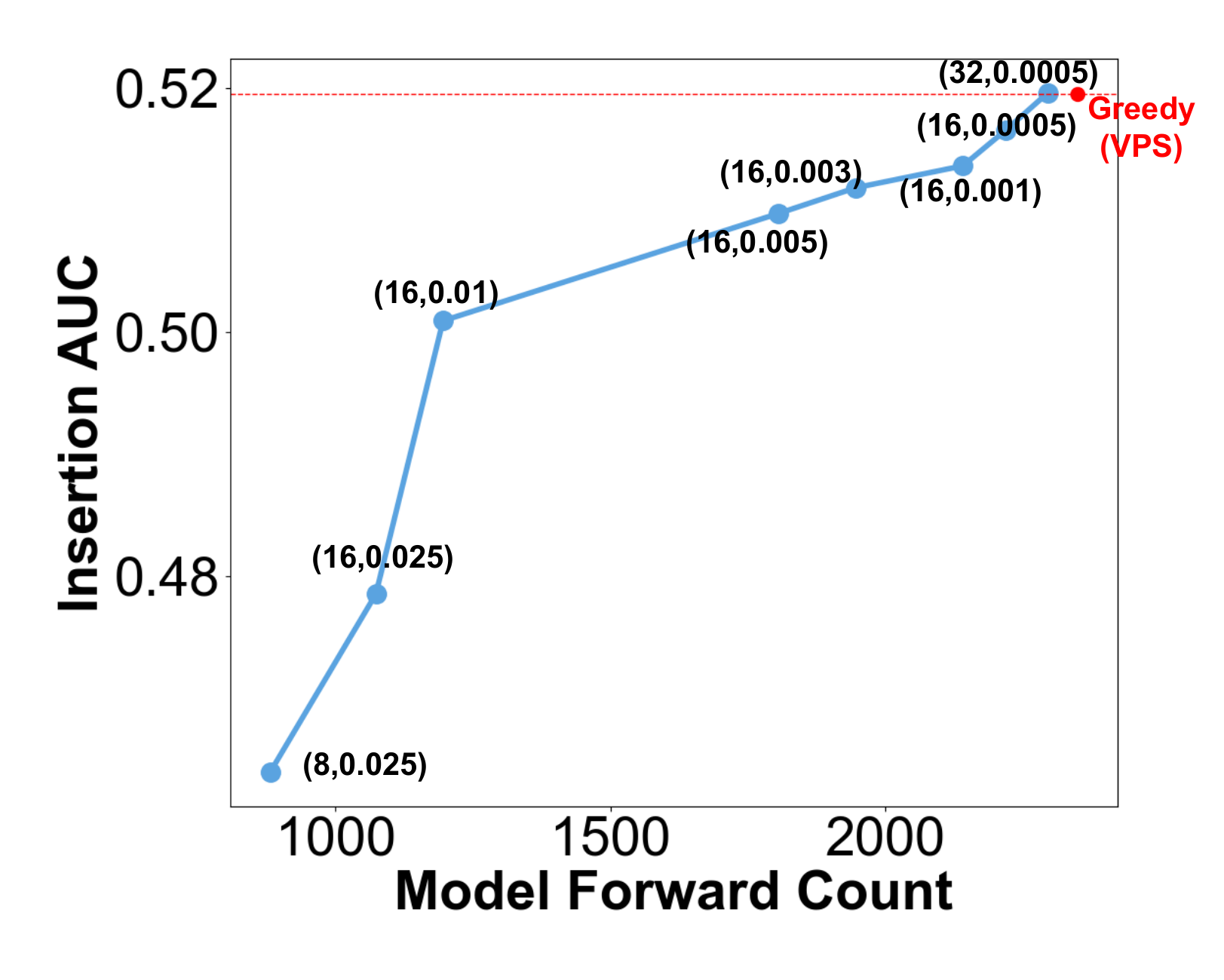}
    \caption{Trade-off between speed and precision.}
    \label{fig:af}
\end{figure}

An appealing property of our acceleration algorithm is the ability to balance efficiency and accuracy
through hyperparameter tuning. By slightly relaxing the speed constraint, PhaseWin-50 can steadily
improve its attribution quality. As shown in Figure~\ref{fig:af}, the insertion AUC increases
monotonically with the number of model forward passes, approaching the performance of the
greedy algorithm. Owing to the annealing strategy, our method can even surpass greedy search when
fully trading off speed, demonstrating that efficiency and precision can be adaptively controlled.

\subsection{Visualization}
We further present visualization results for correctly attributed cases,
As shown in Figure~\ref{fig:vis}, ODAM produces diffuse heatmaps, while D-RISE generates noisy regions due to perturbation sampling. 
VPS (Greedy)-50 yields sharp attributions but at a prohibitive computational cost. 
In contrast, our PhaseWin-50 achieves nearly the same attribution quality with only about 20\% of the overhead, and its annealing strategy often allows the max object score to surpass VPS (Greedy) by better exploring the maximum submodular subset. More visualization results are included in Appendix~\ref{appvis}.

\section{Conclusion}
\label{cls}
In this work, we addressed the challenge of efficient attribution for large multimodal foundation models in object detection. Building on the submodular hypothesis and task-specific properties, we proposed the PhaseWin algorithm as a principled and scalable alternative to the original greedy attribution search. PhaseWin preserves the fidelity advantages of greedy optimization while substantially reducing computational cost, achieving up to 95\% of state-of-the-art attribution fidelity using only 20\% of the evaluations. This establishes PhaseWin as a practical and effective solution for interpreting object-level multimodal models. Beyond object detection, the algorithm’s general design makes it broadly applicable to image-based attribution tasks, and we expect future work to explore its extension to a wider range of multimodal foundation models.

{
    \small
    \bibliographystyle{ieeenat_fullname}
    \bibliography{main}
}

\appendix
\newpage
\section{Window Selection Polices}
\label{appendix:window}
In this section, we first introduce the four algorithms (described in Table~\ref{tab:phasewin_policies}) we can choose from for the sub-process designed for a window of the phasewin algorithm.

First, the most basic approach is to apply greedy search within the window, which is also the slowest. Our three subsequent designs all use the submodularity assumption to varying degrees to reduce the number of searches within the window. $\pi_{BA}$ uses an adaptive scaling search strategy, $\pi_{T_2}$ considers two elements with the smallest reduction in combined return, and $\pi_{BAF}$ reduces the number of comparisons by maintaining an upper bound list.

\begin{table}[ht]
\centering
\caption{Window selection policies $\pi(\cdot)$ used within the \texttt{WindowSelection} subroutine.}
\label{tab:phasewin_policies}
\begin{tabular}{l p{0.7\linewidth}}
\toprule
\textbf{Policy} & \textbf{Description} \\
\midrule
$\pi_{\mathrm{LG}}$ & \textbf{Local-Greedy:} Picks the top candidate if its gain exceeds $\tau_{\mathrm{sel}}$. \\
$\pi_{\mathrm{BA}}$ & \textbf{Beta-Adaptive:} Selects all candidates above an adaptive cut-off based on the window's max gain. \\
$\pi_{\mathrm{T2}}$ & \textbf{Top-2:} Jointly selects the top two candidates if their gains are high and their relative gap is small. \\
$\pi_{\mathrm{BAF-B}}$ & \textbf{Batched Best-Above w/ Forward-checking:} Processes the window in batches, using cached gains to terminate early and reduce evaluations. \\
\bottomrule
\end{tabular}
\end{table}

\section{Complete Algorithm Process}
\label{appendix:process}
The algorithm operates in discrete phases. At the start of each phase, it performs a full evaluation on all remaining candidate regions ($\mathcal{R}$) and greedily selects the single best region to anchor the current search state. This ensures consistent progress. Based on the maximum marginal gain ($G_t$) observed in this step, it computes two adaptive thresholds: a selection threshold $\tau_{\mathrm{sel}}$ to identify high-potential candidates and a pruning threshold $\tau_{\mathrm{del}}$ to discard low-utility regions. This adaptive pruning strategy dynamically narrows the search space, focusing computational resources on the most promising regions.

For the initial phases (controlled by a hyperparameter $m_{\mathrm{active}}$), PhaseWin employs a sliding window of size $w$ over the sorted candidate pool $\mathcal{P}_t$. Within this window, a selection policy $\pi(\cdot)$---such as Beta-Adaptive (BA)---is applied to identify a batch of one or more regions for selection. This allows the model to select complementary regions simultaneously, a capability absent in naive greedy search. To further refine the candidate evaluation, a simulated annealing mechanism may defer the entry of lower-scoring regions into the window, allowing more promising candidates to be assessed first. After $m_{\mathrm{active}}$ phases, the algorithm transitions to a simplified greedy selection over the candidate pool to ensure convergence.

A key innovation of PhaseWin is its \emph{dynamic phase supervision}. We monitor the sequence of marginal gains of the selected regions, $\Delta_i = \mathcal{F}(S_i) - \mathcal{F}(S_{i-1})$. If the current gain drops precipitously compared to the previous one (i.e., $\Delta_i < \theta \cdot \Delta_{i-1}$, where $\theta$ is an adaptive supervision coefficient), it signals a potential breakdown of local submodularity. In this event, the algorithm calculates a probability $p_{\mathrm{exit}}$ to terminate the current phase prematurely. This probabilistic exit prevents the algorithm from wasting evaluations on a sequence of diminishing returns and allows it to restart with a new anchor region. The complete procedure is detailed in Algorithm~\ref{alg:phasewin}.

\begin{algorithm*}[t]
\caption{Phase-Window (PhaseWin) Search Algorithm}
\label{alg:phasewincomplete}
\footnotesize
\DontPrintSemicolon
\SetKwInput{KwIn}{Input}
\SetKwInput{KwOut}{Output}
\SetKwInput{KwHyp}{Hyperparameters}

\KwIn{Region set $\mathcal{V}$, target size $k$, scoring function $\mathcal{F}(\cdot,\mathbf{b}_{\mathrm{target}},c)$}
\KwHyp{Window size $w$, active window phases $m_{\mathrm{active}}$, selection ratio $\alpha_{\mathrm{sel}}$, deletion ratio $\beta_{\mathrm{del}}$, supervision coefficients $\{\theta_t\}$}
\KwOut{Ordered set $S$}

$S \gets \emptyset$,\quad $\mathcal{R} \gets \mathcal{V}$,\quad $t \gets 0$,\quad $\Delta_{\mathrm{prev}} \gets \infty$\;

\While{$|S| < k$ \KwSty{and} $\mathcal{R} \neq \emptyset$}{
    $t \gets t + 1$\;

    Compute $g_r \gets \mathcal{F}(S \cup \{r\})$ for all $r \in \mathcal{R}$\;
    \If{$\max(g_r) \le 0$}{\KwSty{break}}
    
    $\alpha_{\mathrm{best}} \gets \arg\max_{r \in \mathcal{R}} g_r$\;
    $S \gets S \cup \{\alpha_{\mathrm{best}}\}$\;
    $\mathcal{R} \gets \mathcal{R} \setminus \{\alpha_{\mathrm{best}}\}$\;
    $\Delta_t \gets g_{\alpha_{\mathrm{best}}}$\;
    
    Recompute $g_r$ for all $r \in \mathcal{R}$ and let $G_t = \max_{r} g_r$\;
    
    $\tau_{\mathrm{sel}} \gets \alpha_{\mathrm{sel}} \cdot G_t$\;
    $\tau_{\mathrm{del}} \gets \beta_{\mathrm{del}} \cdot G_t$\;
    
    $\mathcal{R} \gets \{ r \in \mathcal{R} ~|~ g_r \ge \tau_{\mathrm{del}} \}$\tcp*{Prune low-gain regions}
    
    $\mathcal{P}_t \gets \{ r \in \mathcal{R} ~|~ g_r \ge \tau_{\mathrm{sel}} \} \cup \text{RandomSample}(\{ r \in \mathcal{R} ~|~ g_r < \tau_{\mathrm{sel}} \})$\;
    
    Sort $\mathcal{P}_t$ in descending order of $g_r$\;

    \If{$t \le m_{\mathrm{active}}$}{
        \tcp{Windowing Mode}
        Initialize window $W$ with top $w$ elements from $\mathcal{P}_t$\;
        
        \While{$|W| > 0$ \KwSty{and} $|S| < k$}{
            $A \gets \pi(W, \mathcal{F}, \tau_{\mathrm{sel}})$\tcp*{Selection policy (e.g., BA)}
            \If{$A = \emptyset$}{\KwSty{break}}
            
            \ForEach{$\alpha \in A$}{
                $\Delta_i \gets \mathcal{F}(S \cup \{\alpha\}) - \mathcal{F}(S)$\;
                
                \If{$\Delta_i < \theta_t \cdot \Delta_{\mathrm{prev}}$}{
                    Compute exit probability $p_{\mathrm{exit}}(\Delta_i, \Delta_{\mathrm{prev}}, \theta_t)$\;
                    \If{$\mathrm{rand}() < p_{\mathrm{exit}}$}{
                        \textbf{goto} end\_phase\;
                    }
                }
                $S \gets S \cup \{\alpha\}$,\quad $\Delta_{\mathrm{prev}} \gets \Delta_i$\;
            }
            
            Remove evaluated items from $W$ and refill from $\mathcal{P}_t$\;
        }
    }
    \Else{
        \tcp{Degenerate Greedy Mode}
        \ForEach{$\alpha \in \mathcal{P}_t$}{
            $\Delta_i \gets \mathcal{F}(S \cup \{\alpha\}) - \mathcal{F}(S)$\;
            
            \If{$\Delta_i < \theta_t \cdot \Delta_{\mathrm{prev}}$}{
                Compute exit probability $p_{\mathrm{exit}}$\;
                \If{$\mathrm{rand}() < p_{\mathrm{exit}}$}{\KwSty{break}}
            }
            $S \gets S \cup \{\alpha\}$,\quad $\Delta_{\mathrm{prev}} \gets \Delta_i$\;
            \If{$|S| \ge k$}{\KwSty{break}}
        }
    }
    
    \tcp*[h]{End current phase}
    \label{alg:phasewin:end_phase}
    end\_phase:\;
}
\Return{$S$}\;
\end{algorithm*}

\section{Evaluation Metrics}
\label{app:metrics}

\paragraph{Faithfulness.} 
To assess how well an attribution map reflects the model's reasoning, we compute the Insertion and Deletion AUC scores, which quantify the change in model output as the most (or least) important superpixels are progressively revealed or removed~\cite{petsiuk2018rise}. We apply these metrics both to classification confidence and to Intersection-over-Union (IoU), thus capturing the attribution's influence on recognition and localization. We further measure the highest confidence score for any predicted box with $\mathrm{IoU} > 0.5$ relative to the target. For failure cases, we evaluate the \textit{Explaining Successful Rate (ESR)}, which measures whether a saliency map can guide the model to a correct detection for initially misclassified or low-confidence predictions. 

\paragraph{Localization Accuracy.} 
We use two established metrics: (i) the \textit{Point Game}, which checks whether the most salient pixel lies inside the ground-truth bounding box, and (ii) the \textit{Energy Point Game}, which extends this by considering the energy concentration of saliency around the target~\cite{zhang2018top}. These metrics are evaluated only on correctly detected samples.

\paragraph{Efficiency.} 
To provide a fair and hardware-agnostic cost measure, we introduce the \textit{Model Evaluation Count (MEC)} as our primary efficiency metric, where one unit corresponds to a single forward pass through the model. The total MEC reflects the algorithm’s runtime cost. Additionally, we define the \textit{Accuracy–Cost Ratio (AC-Ratio)} as the primary performance metric (faithfulness score) multiplied by 1000 and divided by the MEC. This ratio is most meaningful when the faithfulness score meets a predefined quality threshold.

\section{Implementation Details}
\label{app:impl}

In all experiments, the ground-truth bounding box $b_{\mathrm{target}}$ and its category $c$ are provided as references for generating attributions. Each image is segmented into 100 sparse sub-regions using the SLICO superpixel algorithm, which serve as the interpretable units.  

For PhaseWin, we apply a window size of 16 when selecting from 50 sub-regions and 32 when selecting from 100 sub-regions. Results are averaged over five random seeds, with variance consistently below 2\%.  

As the scoring function is not strictly monotonic submodular, the stopping criterion is implemented in a ratio-based form:  
\[
\frac{S_{k-2}}{S_{k-1}} - \frac{S_{k-1}}{S_{k}} \leq \tau.
\]  
We use $\tau=0.025$ for 50 sub-regions and $\tau=0.01$ for 100 sub-regions. This criterion ensures numerical stability across different settings. 
\FloatBarrier
\newpage
\onecolumn
\section{Full Proof}
\label{appendix:proof}
In this section, we will introduce the proof of Theorem~\ref{thm}. Property~\ref{pro} is a classic result of combinatorial optimization. If you are interested in Property 3.1, you can find the relevant proof in ~\cite{edmonds1970submodular,fujishige2005submodular}.

\textbf{Proof of Theorem~\ref{thm}.}

\begin{proof}
If the parameter for AdaptiveThreshold is $(\alpha,\gamma)$ (for select and delete), the parameter for WindowSelection  when $|S|=i$ is $\beta_i$ with $\beta_i$ increasing and $\alpha\beta_i\geq \gamma$.

Let $S_{\text{PhaseWin}}=(v_1,v_2,\ldots,v_k)$, $S_0=\emptyset$ and $S_i=\{v_1,v_2,\ldots,v_i\}$.
Let $\rho_u(V)=\mathcal{F}(V\cup \{u\})-\mathcal{F}(V)$.

For each $1\leq i\geq k$ such that $v_i$ is an element directly added into $S_{\text{PhaseWin}}$ without going into WindowSelection, let $\mathcal{R}_i$ to be the set of choosable elements before $v_i$ is selected, $\mathcal{D}_i$ to be the set of deleted elements after $v_i$ is selected in PartitionCandidates.
Then we have
\begin{align*}
    a_i&\triangleq\rho_{v_i}(S_{i-1}) =\max_{j\in \mathcal{R}_i} \rho_j(S_{i-1}),\\
    \mathcal{D}_i&\triangleq\{j\in \mathcal{R}_i\mid \rho_j(S_{i-1})<\gamma a_i\},\\
    V_i&\triangleq\{j\in \mathcal{R}_i\mid \rho_j(S_{i-1})>\alpha a_i\};\\
    W_i&\triangleq(e_{i,1},e_{i,2},\ldots,e_{i,m_i})\subseteq V_i\text{ is the maximum sequence such that}\\
    e_{i,j}&=v_{i+j}=\text{argmax}\{\rho_e(S_{i+j-1}) \mid e\in S_i\setminus\{e_{i,1},\ldots, e_{i,j-1}\}\} \text{ and }\\
    b_{i,j}&\triangleq \rho_{e_{i,j}}(S_{i+j-1})\geq \beta_{i+j}b_{i,0}\geq \alpha\beta_{i+j}a_i\geq \alpha\beta_{i+j}\max_{j\in \mathcal{R}_{i}} \rho_j(S_{i+j-1}).
\end{align*}
Thus for any $1\leq l\leq k$, we have $$\rho_{v_l}(S_{l-1})\geq \alpha\beta_{l}\max_{j\in \mathcal{R}_l}\rho_j(S_{l-1}).$$

Since $\mathcal{F}$ is increasing and submodular, for any $1\leq l\leq k$ we have
\begin{align*}
    \zopt&\leq \mathcal{F}(S_{l-1})+\sum_{j\in (T\setminus S_{l-1})\cap \mathcal{R}_l} \rho_j(S_{l-1})+\sum_{m=1}^{l-1}\sum_{x\in (T\setminus S_{l-1})\cap D_m}\rho_j(S_{l-1})\\
    &\leq \mathcal{F}(S_{l-1})+\dfrac{k}{\alpha\beta_l}\rho_{v_l}(S_{l-1})+k\gamma\sum_{m=1}^{l-1}\rho_{v_m}(S_{m-1})\\
    &=\dfrac{k}{\alpha\beta_l}\mathcal{F}(S_l)-(\dfrac{k}{\alpha\beta_l}-1-k\gamma)\mathcal{F}(S_{l-1}).
\end{align*}
Let $\lambda_i=\dfrac{\alpha\beta_i}{k}$ and $\mu_i=\dfrac{\alpha\beta_i}{k}(\dfrac{k}{\alpha\beta_i}-1-k\gamma)$, then we have
$$ \zpw \geq \zopt\cdot(\lambda_k+\mu_k\lambda_{k-1}+\mu_k\mu_{k-1}\lambda_{k-2}+\cdots +\mu_k\mu_{k-1}\ldots\mu_2\lambda_1).$$
In particular, if $\beta_i=\beta$ for $i=1,2,\ldots,k$, then
$$\zpw\geq \dfrac{\lambda_1(1-\mu_1^k)}{1-\mu_1}\zopt.$$
If $k,\alpha,\beta$ is big enough and $\gamma$ is small enough, then
$$\zpw\geq (1-\dfrac{1}{e}-o(1))\zopt.$$
\end{proof}

\section{Submodularity and Supermodularity}
\label{appendix:submodularity}

\subsection{Definitions}
Let $V$ denote a finite ground set of candidate regions and $F: 2^V \to \mathbb{R}$ be a set function that scores any subset $S \subseteq V$.

\begin{definition}[Submodularity]
A set function $F$ is \emph{submodular} if it satisfies the \emph{diminishing returns property}: for all $A \subseteq B \subseteq V$ and $\alpha \in V \setminus B$,
\[
F(A \cup \{\alpha\}) - F(A) \;\;\ge\;\; F(B \cup \{\alpha\}) - F(B).
\]
That is, the marginal gain of adding an element decreases as the context grows.
\end{definition}

\begin{definition}[Supermodularity]
A set function $F$ is \emph{supermodular} if it satisfies the \emph{increasing returns property}: for all $A \subseteq B \subseteq V$ and $\alpha \in V \setminus B$,
\[
F(A \cup \{\alpha\}) - F(A) \;\;\le\;\; F(B \cup \{\alpha\}) - F(B).
\]
That is, the marginal gain of adding an element increases as the context grows.
\end{definition}

\subsection{Optimization Significance}
Submodularity generalizes the notion of convexity to discrete set functions. Maximizing a monotone submodular function under a cardinality constraint admits a simple greedy algorithm with a $(1 - 1/e)$-approximation guarantee, which is provably optimal under polynomial-time complexity assumptions. In contrast, supermodular functions exhibit cooperative effects, and their maximization is generally intractable, while their minimization is often easier.

\subsection{AUC Curve Properties}
In attribution evaluation, we consider the insertion process: progressively adding sub-regions $s_1, s_2, \dots$ into the image. Let
\[
\mathrm{AUC}(k) = \frac{1}{|V|} \sum_{j=1}^k F(\{s_1,\dots,s_j\})
\]
denote the cumulative insertion-AUC score up to step $k$.

\begin{theorem}
If $F$ is submodular, then the insertion AUC curve $\mathrm{AUC}(k)$ is concave in $k$.  
If $F$ is supermodular, then $\mathrm{AUC}(k)$ is convex in $k$.
\end{theorem}

\begin{proof}[Sketch]
For submodular $F$, diminishing returns imply that the marginal gain $F(S \cup \{s\}) - F(S)$ is non-increasing in $|S|$. Thus, the discrete derivative of $\mathrm{AUC}(k)$ decreases with $k$, yielding concavity. Conversely, if $F$ is supermodular, marginal gains increase with $k$, so $\mathrm{AUC}(k)$ is convex.
\end{proof}

\subsection{Implications for Deep Learning}
Deep neural networks do not strictly satisfy either submodularity or supermodularity. Instead, their attribution behavior reflects a hybrid of both: some features exhibit redundancy (submodular-like), while others rely on synergy (supermodular-like). From the perspective of distributed computation, submodularity and supermodularity describe not universal properties of the model but rather the modes of feature aggregation. Submodularity corresponds to robust, redundant feature usage, while supermodularity corresponds to cooperative, highly interactive feature combinations. These patterns shed light on how models internally organize basic feature units, rather than providing exact guarantees.

The two models we selected are, respectively, dominated by submodularity and supermodularity. Below, we show the Insertion AUC curves (Figure~\ref{sub}) for Grounding DINO and Florence-2 on the same sample. Their unevenness indicates that Grounding DINO exhibits submodularity most of the time, while Florence-2 is almost universally submodular. Our algorithm achieved acceleration on both models, and the difference in performance is precisely due to the difference between submodularity and supermodularity.

\begin{figure}[ht]
    \centering
    \includegraphics[width=0.9\linewidth]{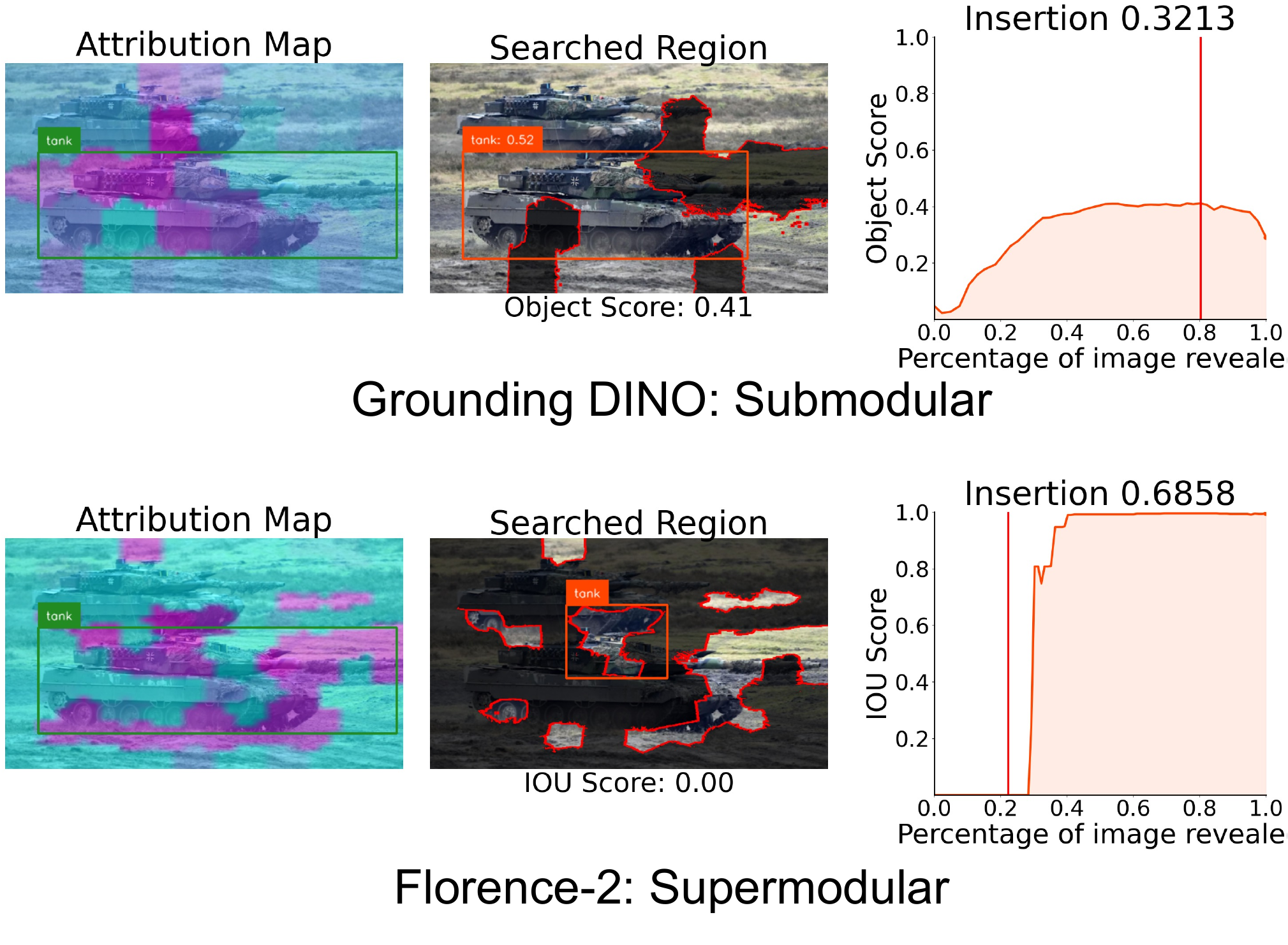}
    \caption{Insertion AUC under Greedy (VPS). Grounding DINO is almost concave with only a few exceptions, while Florence-2 is completely convex.}
    \label{sub}
\end{figure}
\FloatBarrier
\newpage

\section{Additional Visualization Results}
\label{appvis}

In this section, we provide additional qualitative results to further illustrate the visual differences between the original Visual Precision Search (VPS) and our PhaseWin. Each figure presents one representative example from different tasks and datasets. For each case, we show side-by-side attribution maps highlighting how both methods localize critical regions that drive the prediction of object-level foundation models. These examples complement the quantitative results in Section~\ref{exp}, demonstrating that PhaseWin preserves interpretability quality while achieving substantial efficiency gains.

\begin{figure}[ht]
    \centering
    \includegraphics[width=0.9\linewidth]{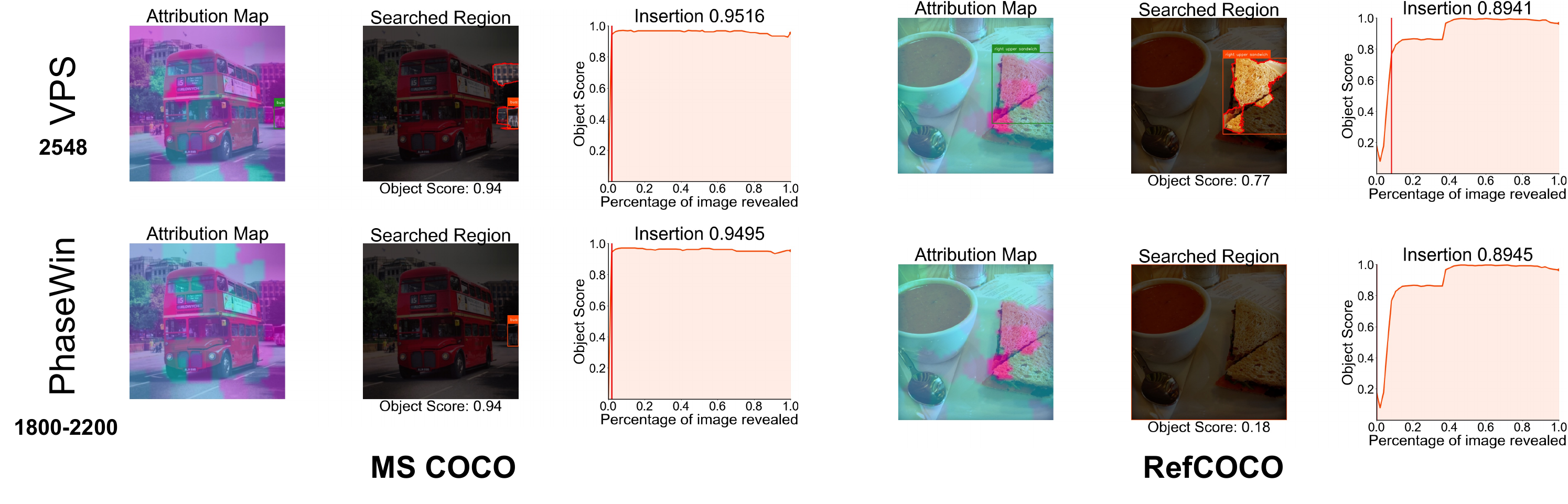}
    \caption{Comparison between VPS and PhaseWin on Florence-2 for correctly detected samples in MS COCO and RefCOCO. Both methods highlight semantically relevant regions, while PhaseWin produces equally faithful maps with far fewer evaluations.}
    \label{fig:flovis}
\end{figure}

\begin{figure}[ht]
    \centering
    \includegraphics[width=0.9\linewidth]{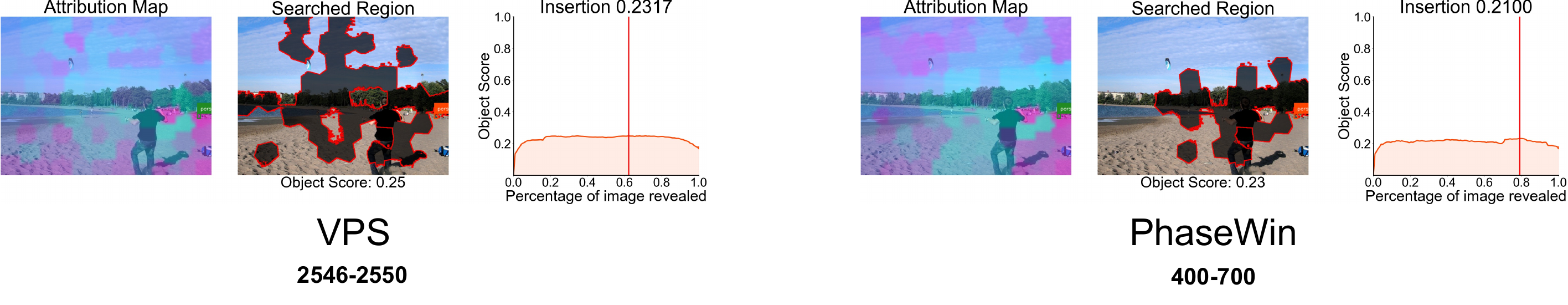}
    \caption{Visualization on Grounding DINO (MS COCO misclassification). VPS and PhaseWin consistently attribute the incorrect prediction to the same misleading region, confirming that PhaseWin maintains interpretive fidelity even in failure cases.}
    \label{fig:grmiscoco}
\end{figure}

\begin{figure}[ht]
    \centering
    \includegraphics[width=0.9\linewidth]{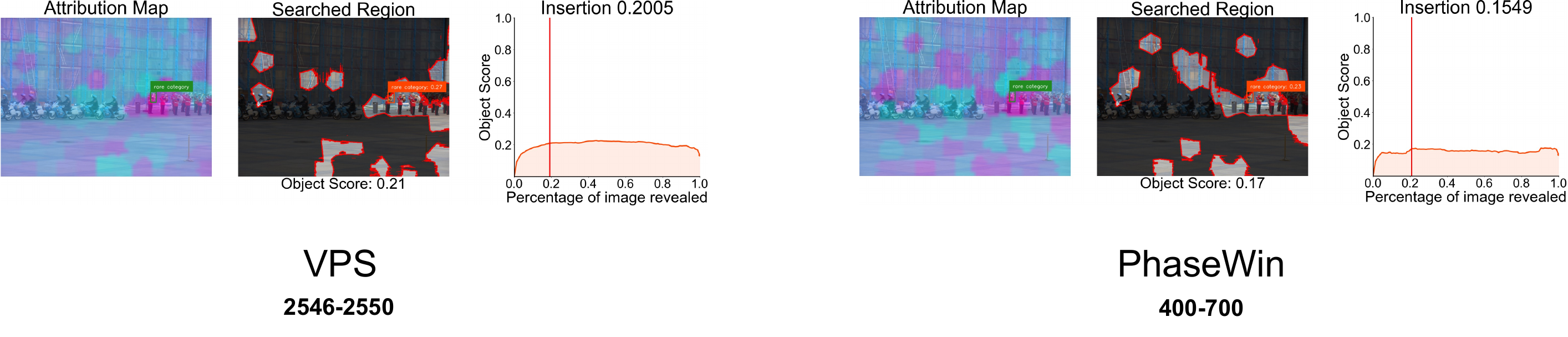}
    \caption{Visualization on Grounding DINO (LVIS misclassification). Both methods reveal the background context responsible for confusion, with PhaseWin matching the fine-grained localization quality of VPS at lower computational cost.}
    \label{fig:grmislvis}
\end{figure}

\begin{figure}[ht]
    \centering
    \includegraphics[width=0.9\linewidth]{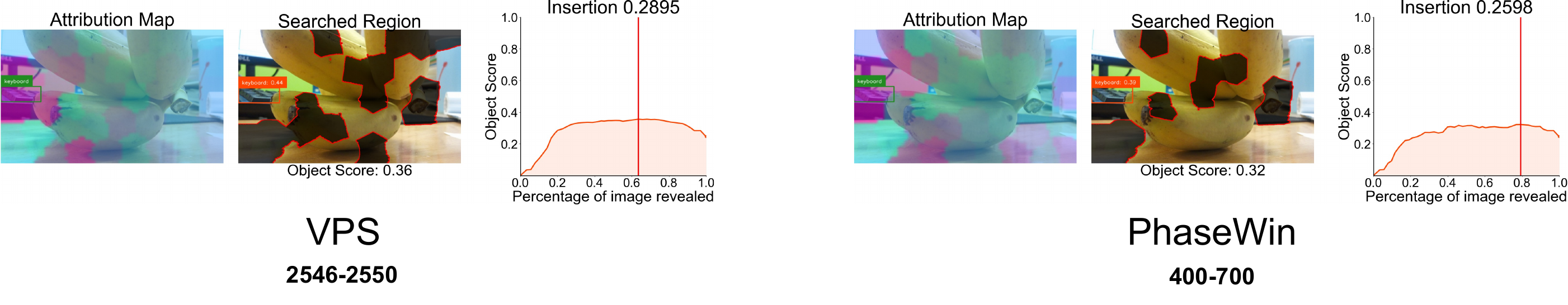}
    \caption{Visualization on Grounding DINO (MS COCO missed detection). VPS and PhaseWin identify the overlooked object region. PhaseWin effectively reproduces the trajectory of evidence accumulation with a fraction of the overhead.}
    \label{fig:grcocodet}
\end{figure}

\begin{figure}[ht]
    \centering
    \includegraphics[width=0.9\linewidth]{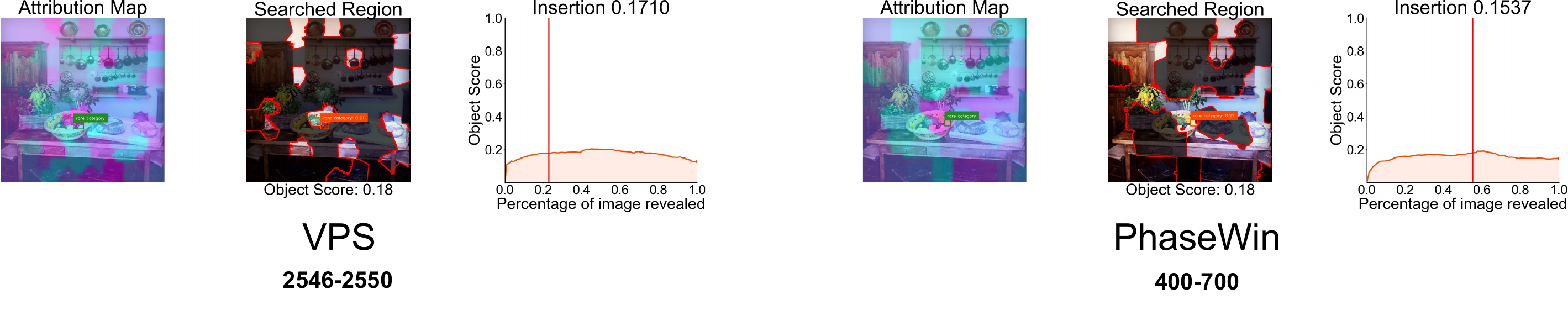}
    \caption{Visualization on Grounding DINO (LVIS missed detection). PhaseWin successfully recovers the same key evidence regions highlighted by VPS, showing its robustness on challenging zero-shot categories.}
    \label{fig:grlvvis}
\end{figure}

\begin{figure}[ht]
    \centering
    \includegraphics[width=0.9\linewidth]{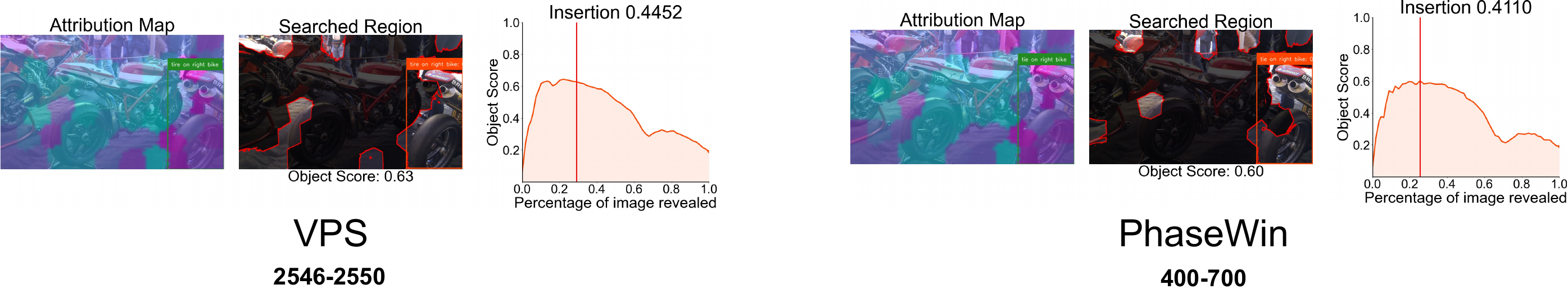}
    \caption{Visualization on Grounding DINO (RefCOCO grounding mistake). Both methods attribute the grounding failure to distractor regions, while PhaseWin provides nearly identical explanations with significantly fewer model evaluations.}
    \label{fig:grrfvis}
\end{figure}

\end{document}